\documentclass[runningheads]{llncs}
\usepackage[T1]{fontenc}
\usepackage{graphicx}

\usepackage{amsmath,amsfonts,bm}

\def\eqref#1{equation~\ref{#1}}

\def\1{\bm{1}}

\def\rvv{{\mathbf{v}}}

\def\vtheta{{\bm{\theta}}}

\def\vv{{\bm{v}}}

\def\vx{{\bm{x}}}
\def\vy{{\bm{y}}}
\def\vz{{\bm{z}}}

\def\mA{{\bm{A}}}

\def\mG{{\bm{G}}}

\def\mI{{\bm{I}}}

\DeclareMathAlphabet{\mathsfit}{\encodingdefault}{\sfdefault}{m}{sl}
\SetMathAlphabet{\mathsfit}{bold}{\encodingdefault}{\sfdefault}{bx}{n}

\def\emG{{G}}

\newcommand{\E}{\mathbb{E}}

\newcommand{\Var}{\mathrm{Var}}

\DeclareMathOperator{\Tr}{Tr}

\usepackage{hyperref}
\usepackage{xurl}
\usepackage{epstopdf}
\usepackage{algorithm}
\usepackage{algpseudocode}
\usepackage{subcaption}
\usepackage{amsmath, amssymb}
\usepackage[svgnames]{xcolor}
\usepackage{colortbl}
\usepackage{adjustbox}
\usepackage{booktabs}
\usepackage{microtype}
\usepackage[shortlabels]{enumitem}
\usepackage{adjustbox}
\usepackage{makecell}
\usepackage{multirow}
\usepackage{caption}
\usepackage{wrapfig}

\usepackage{tikzit}

\tikzstyle{circled}=[fill=white, draw=black, shape=circle, left, minimum size=1pt]
\tikzstyle{dot}=[fill=black, draw=black, shape=circle, inner sep=1.0pt, outer sep=1.5pt]
\tikzstyle{circle shape}=[fill=none, draw=black, shape=circle]

\tikzstyle{arrow}=[draw=black, ->, line width=0.2mm]
\tikzstyle{double arrow}=[<->]
\tikzstyle{dashedline}=[-, dashed, line width=0.2mm]
\tikzstyle{regular line}=[-, line width=0.2mm]
\tikzstyle{dashed arrow}=[dashed, line width=0.2mm, ->]
\tikzstyle{new edge style 0}=[-, line width=0.2mm, fill={rgb,255: red,208; green,208; blue,208}]

\usepackage{color}

\newcommand{\xx}{\vx^\ast}
\newcommand{\dt}{\partial_t}
\newcommand{\tr}{\mathrm{Tr}}
\newcommand{\ef}{f^{emp}}

\newcommand\Tstrut{\rule{0pt}{2.6ex}}
\newcommand{\appropto}{\mathrel{\vcenter{
  \offinterlineskip\halign{\hfil$##$\cr
    \propto\cr\noalign{\kern2pt}\sim\cr\noalign{\kern-2pt}}}}}

\def\mGamma{{\bm{\Gamma}}}
\def\X{X}

\newif\iffull

\fulltrue

\begin{document}

\title{TULiP: Test-time Uncertainty Estimation via Linearization and Weight Perturbation}
\titlerunning{TULiP: Test-time Uncertainty Estimation}

\author{Yuhui Zhang\inst{1} \and
Dongshen Wu\inst{2} \and
Yuichiro Wada\inst{3,4} \and
Takafumi Kanamori\inst{1,4}
}

\institute{
Institute of Science Tokyo, Japan \email{kanamori@c.titech.ac.jp} \and
Imperial College London, UK \and
Fujitsu Limited, Japan \and
RIKEN AIP, Japan
}

\maketitle

\begin{abstract}
A reliable uncertainty estimation method is the foundation of many modern out-of-distribution (OOD) detectors, which are critical for safe deployments of deep learning models in the open world.
In this work, we propose TULiP, a theoretically-driven post-hoc uncertainty estimator for OOD detection.
Our approach considers a hypothetical perturbation applied to the network before convergence. Based on linearized training dynamics, we bound the effect of such perturbation, resulting in an uncertainty score computable by perturbing model parameters.
Ultimately, our approach computes uncertainty from a set of sampled predictions.
We visualize our bound on synthetic regression and classification datasets.
Furthermore, we demonstrate the effectiveness of TULiP using large-scale OOD detection benchmarks for image classification. Our method exhibits state-of-the-art performance, particularly for near-distribution samples.
\iffull
\else
\footnote{Further details are available in~\cite{tuliparxiv}.}
\fi
\end{abstract}

\section{Introduction}
\label{sec: introduction}
An important safety component for deep neural networks (NNs) in real-world environments is the awareness of their uncertainty upon receiving unknown or corrupted inputs. Such capability enables systems to fall back to conservative decision-making or defer to human judgments when faced with unfamiliar scenarios, which is imperative in safety-critical domains, such as autonomous driving~\cite{DBLP:journals/access/AtakishiyevSYG24} and medical applications~\cite{esteva2017dermatologist}. The problem is often framed as \textbf{Out-Of-Distribution (OOD)} detection, which has witnessed significant growth in recent years~\cite{yang2024generalized}.

Theoretically, this issue directly relates to quantifying epistemic uncertainty~\cite{hora1996aleatory}, which measures the lack of knowledge in a fitted model due to insufficient training data.
The training process is typically modelled as a Bayesian optimization process~\cite{DBLP:journals/csur/WangY20} with approximations for practical use~\cite{gal2016dropout,DBLP:conf/nips/DaxbergerKIEBH21}.
More generally, epistemic uncertainty could be formalized by the variance of a trained ensemble of networks $\phi(\vx; \vtheta)$:%
\begin{equation}
    \label{eq: GT-uncertainty}
    \Var_{\vtheta_\mathrm{Init}}\left[ \phi(\vx; \vtheta_\mathrm{Train}) \right],
\end{equation}
where $\vtheta_\mathrm{Train}$ are parameters trained by some learning algorithm from random initialization $\vtheta_\mathrm{Init}$.
Intuitively, higher prediction variance corresponds to inputs $\vx$ further from training set (OOD), as there lack enough training data to eliminate model disagreements via training, hence epistemic.

Many works redesign the network or training process to be uncertainty-aware~\cite{devries2018learning,huang2021mos}. However, these are often impractical due to heavy computational costs, especially for large datasets. Instead,
\emph{post-hoc} methods~\cite{liang2020odin,liu2021energy,mls2022icml,djurisic2023ash} are generally preferred.
These approaches can be easily integrated into pre-trained models without interfering with the trained backbones, significantly enhancing their versatility~\cite{yang2022openood}.
Nevertheless, they often lack a direct theoretical link to the training process, which weakens their theoretical foundation and necessitates extensive empirical validation.

Therefore, it is desirable to develop a post-hoc OOD method with direct theoretical justifications regarding the training process. Recent analysis of NN optimizations reveals that gradient descent can be seen as its first-order approximations~\cite{jacot2018neural,lee2019wide}, termed \emph{lazy} training, under specific conditions~\cite{geiger2020disentangling}.
This enabled direct (but costly) computation of Eq.~\ref{eq: GT-uncertainty}, as well as rigorous analysis~\cite{kobayashi2022disentangling} and methods~\cite{he2020bayesian} on model uncertainty, even beyond the lazy regime~\cite{chen2020dynamical}.

Inspired by this series of work, we present \textbf{TULiP} (\textbf{T}est-time \textbf{U}ncertainty by \textbf{Li}nearized fluctuations via weight \textbf{P}erturbation), a post-hoc uncertainty estimator for OOD detection. Our method considers hypothetical fluctuations of the lazy training dynamics, which can be bounded under certain assumptions and efficiently estimated via weight perturbation. In practice, we found our method works well even beyond the ideal regime. Our contribution is threefold: 
\begin{enumerate}[(i)]
    \item We provide a simple, versatile theoretical framework for analyzing epistemic uncertainty at inference time in the lazy regime, which is empirically verified;
    \item Based on our theory, we propose TULiP, an efficient and effective post-hoc OOD detector that does not require access to original training data;
    \item We test TULiP extensively using OpenOOD~\cite{zhang2023openoodv15}, a large, transparent, and unified OOD benchmark for image classifications. We show that TULiP consistently improves previous state-of-the-art methods across various settings.
\end{enumerate}

The outline is as follows. Sec.~\ref{sec: related} provides a summary of related works, Sec.~\ref{sec: method} presents theoretical derivations, and Sec.~\ref{sec: impl} bridges theory to the implementation of TULiP. Sec.~\ref{sec: experiments} reports the effectiveness of TULiP via empirical studies.

\section{Related Works}
\label{sec: related}
\subsubsection{Uncertainty Quantification (UQ)}

As being discussed in Sec.~\ref{sec: introduction}, theoretically-driven methods often estimates epistemic uncertainty from a Bayesian perspective.
This includes, notably, Variational Inference~\cite{DBLP:conf/icml/BlundellCKW15}.
Monte Carlo (MC) Dropout~\cite{gal2016dropout} connects Bayesian inference and the usage of Dropout layers, and is widely adopted in practice due to its simplicity and effectiveness.
Moreover, Laplacian Approximation~\cite{DBLP:conf/nips/DaxbergerKIEBH21} approximates the posterior via Taylor expansion and Deep Ensembles~\cite{DBLP:conf/nips/Lakshminarayanan17} directly used independently trained deep models as an ensemble.

\subsubsection{Post-hoc OOD Detectors}
For post-hoc methods, the baseline method using maximum softmax probability (MSP) was first introduced by \cite{hendrycks2016baseline}. ODIN~\cite{liang2020odin} applies input preprocessing on top of temperature scaling~\cite{guo2017calibration} to enhance MSP. \cite{liu2021energy} proposes a simple score based on energy function (EBO). \cite{mls2022icml} uses maximum logits (MLS) for efficient detection on large datasets. GEN~\cite{10203747} adopts the generalization of Shannon Entropy, while ASH~\cite{djurisic2023ash} prunes away samples’ activation at later layers and simplifies the rest. Some methods also access the training set for additional information, as MDS~\cite{lee2018mds} used Mahalanobis distance with class-conditional Gaussian distributions, and ViM~\cite{wang2022vim} computes the norm of the feature residual on the principal subspace for OOD detection.

Due to the nature of post-hoc setting, most methods such as EBO, ODIN and MLS compute OOD score solely from trained models, overlooking the training process. In contrast, as previously stated, inspired by the more theoretically-aligned UQ methods, TULiP addresses the problem with regard to the training process from a theoretical aspect. In practice, TULiP works by a series of carefully constructed weight perturbations, ultimately yielding a set of model predictions, which can be seen as surrogates to posterior samples for OOD detections. Our contribution is orthogonal to methods working with logits and predictive probabilities, such as GEN, as they can work on top of TULiP outputs. In such an aspect, TULiP shares the similar plug-and-play versatility as seen in recent works, such as ReAct~\cite{NEURIPS2021_01894d6f} and RankFeat~\cite{NEURIPS2022_71c9eb09}.

\section{Theoretical Framework}
\label{sec: method}

\subsection{Preliminaries: Linearized Training Dynamics}
Theories involving Neural Tangent Kernel (NTK) consider the linearization of neural networks.
It has been shown that under an infinite width (lazy) limit, network parameters and hence the gradients barely change across the whole training process, justifying the linearization of the training process~\cite{jacot2018neural}.
\cite{lee2019wide} extends the result by examining them in the parameter space, with a formal result equalizing linearized networks and empirical ones under mild assumptions.

Let $f_\mathrm{True}(\vx;\vtheta): \mathbb{R}^d \rightarrow \mathbb{R}^o$ be a neural network parameterized by parameters $\vtheta$.
We write the Jacobian (gradient) evaluated at $\vx$ as $\nabla_\vtheta f_\mathrm{True}(\vx) \in \mathbb{R}^{o \times |\vtheta|}$, where $|\vtheta|$ is the cardinality of $\vtheta$, i.e., the number of parameters in the network.

Let $f(\vx; \vtheta)$ denote the network linearized at $\vtheta^\ast$:
\begin{equation}
    \label{eq: linearized-net}
    f(\vx; \vtheta) := f_\mathrm{Init}(\vx) + \left. \nabla_\vtheta f_\mathrm{True}(\vx) \right|_{\vtheta = \vtheta^\ast} (\vtheta - \vtheta^\ast),
\end{equation}
where $f_\mathrm{Init}(\vx)$ is the initial network function.
Here, we treat it as a linear approximation to the true training dynamics.
For our convenience, we will interchangeably use $\nabla_\vtheta f(\vx)$ as $\left. \nabla_\vtheta f_\mathrm{True}(\vx) \right|_{\vtheta = \vtheta^\ast}$.

We consider the training data $\vx$ within an empirical dataset $X$.
For a twice-differentiable loss function $\ell(f(\vx); y(\vx))$ with target $y(\vx)$, we write it's gradient w.r.t. $f(\vx)$ as $\ell'(f(\vx); y(\vx))$ (or simply $\ell'(f(\vx))$).
Then, following \cite{lee2019wide}, $f$ is trained on $X$ following the gradient flow:
\begin{equation}
    \label{eq: Linearized dynamics}
    \dt f_t(\vx) = -\eta \E_{\vx'} \left[ \Theta(\vx, \vx') \ell'(f_t(\vx'); y(\vx')) \right],
\end{equation}
where $\E_{\vx'}$ is the expectation w.r.t. the empirical distribution for $\vx' \in X$, $\eta$ is the learning rate and $f_t$ denotes the network $f$ at time $t \in [0, T]$.
Given inputs $\vx, \vx'$, the NTK $\Theta(\vx, \vx') \in \mathbb{R}^{o \times o}$ defined as $\Theta(\vx, \vx') := \nabla_\vtheta f(\vx) \nabla_\vtheta f(\vx')^\top$ governs the linearized training Eq.~\ref{eq: Linearized dynamics}.
Under the lazy limit, the NTK $\Theta(\vx, \vx')$ stays constant across the training process and hence is independent of $t$.
Hereon, we assume the unique existence of the solution to Eq.~\ref{eq: Linearized dynamics}.

\setlength{\intextsep}{0pt}
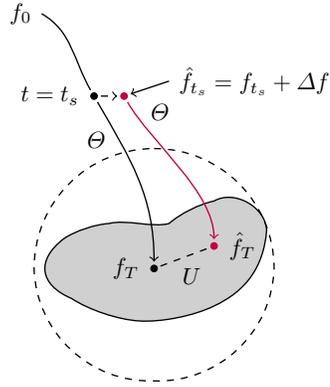
\begin{wrapfigure}{r}{0.4\textwidth}
    \centering
    \resizebox{\linewidth}{!}{
    \begin{tikzpicture}[scale=0.9]
	\begin{pgfonlayer}{nodelayer}
		\node [style=none, label={left:$f_0$}] (0) at (-1.75, 2.75) {};
		\node [style=dot, label={left:$t = t_s$}] (1) at (0, 0) {};
		\node [style=dot, label={left: $f_T$}] (2) at (2, -5.75) {};
		\node [style=dot, color=purple, label={right: $\hat{f}_T$}] (3) at (4, -5) {};
		\node [style=none, label={right: $\hat{f}_{t_s} = f_{t_s} + \Delta f$}] (4) at (2.5, 0.5) {};
		\node [style=none] (6) at (0.075, -1.475) {$\Theta$};
		\node [style=none] (7) at (2.2, -0.55) {$\Theta$};
		\node [style=none] (9) at (3.25, -6) {$U$};
		\node [style=dot, color=purple] (10) at (1, 0) {};
		\node [style=none] (11) at (-1, -4.75) {};
		\node [style=none] (12) at (0.5, -4.25) {};
		\node [style=none] (13) at (2.5, -4.25) {};
		\node [style=none] (14) at (4, -3.5) {};
		\node [style=none] (15) at (5.75, -4.5) {};
		\node [style=none] (16) at (4, -7) {};
		\node [style=none] (17) at (1.5, -7.5) {};
		\node [style=none] (18) at (-0.75, -6.75) {};
		\node [style=none] (19) at (-2, -5.75) {};
		\node [style=none] (20) at (2, -1.75) {};
		\node [style=none] (21) at (6, -5.75) {};
		\node [style=none] (22) at (2, -9.5) {};
	\end{pgfonlayer}
	\begin{pgfonlayer}{edgelayer}
		\draw [style=new edge style 0] (11.center)
			 to [bend left=15, looseness=0.75] (12.center)
			 to [in=-165, out=15, looseness=0.75] (13.center)
			 to [bend left=15, looseness=0.75] (14.center)
			 to [bend left=60] (15.center)
			 to [bend left] (16.center)
			 to [bend left=15, looseness=0.75] (17.center)
			 to [bend left=15, looseness=0.75] (18.center)
			 to [bend left=60, looseness=1.50] cycle;
		\draw [style=dashedline] (20.center)
			 to [bend left=45] (21.center)
			 to [bend left=45] (22.center)
			 to [bend left=45] (19.center)
			 to [bend left=45] cycle;
		\draw [style=regular line, in=120, out=-30] (0.center) to (1);
		\draw [style=arrow, in=90, out=-60] (1) to (2);
		\draw [style=dashedline] (2) to (3);
		\draw [style=arrow] (4.center) to (10);
		\draw [style=dashed arrow] (1) to (10);
		\draw [style=arrow, color=purple, in=90, out=-60, looseness=0.75] (10) to (3);
	\end{pgfonlayer}
\end{tikzpicture}
    }
    \caption{Illustration of the hypothetical perturbation process. The shaded area indicates the distribution of $\hat{f}_T$, and the dashed circle is the Surrogate Posterior Envelope (SPE). See Sec.~\ref{subsec: surrogate-ensemble} for details.}
    \label{fig: illustration-tulip}
\vspace{-2.5em}
\end{wrapfigure}

\subsubsection{Notations}

Let $\vz \in \mathbb{R}^d$ be an arbitrary test point.
Let $\| \cdot \|$ denote the Euclidean norm and induced 2-norm for vectors and matrices.
Let $\| \cdot \|_\mathrm{F}$ denote the matrix Frobenius norm.
We also denote $\left\| \cdot \right\|_{\X} := \E_{\vx} \left[ \|\cdot\|^2 \right]^{1/2}$ the data-dependent norm through out the following descriptions.
For a matrix $\mA$, let $\left| \mA \right|$ be the unique symmetric positive semi-definite solution of 
    $\left| \mA \right|^2 = \mA^\top \mA.$
    It is an extension of absolute values to matrices.
Finally, let $f(\vz) \lesssim g(\vz)$ indicate $f(\vz) \leq g(\vz) + M$ up to some constant $M$ independent of $\vz$.

\subsection{Modeling Uncertainty}

Under our problem setting, neither the distribution of initialized models nor the training process is accessible, which renders a significant difficulty for the direct computation of the uncertainty shown in Eq.~\ref{eq: GT-uncertainty}.
Instead, we choose to intuitively model it by considering a perturbation applied towards the network function $f(\vx)$, at a time $t = t_s$ \emph{before} the training terminates at $t = T$.
It is important to note that, this perturbation prior to convergence is \textit{hypothetical}, as it is inaccessible in our post-hoc setting, and we will only use it to establish our theoretical framework.

Formally, consider a perturbation to $f_{t_s}$ at $t = t_s$ as $\hat{f}_{t_s}(\vx) = f_{t_s}(\vx) + \Delta f(\vx)$.
After the perturbation, the perturbed network $\hat{f}(\vx)$ will be trained following the same dynamics as Eq.~\ref{eq: Linearized dynamics}:
\begin{equation}
    \label{eq: Linearized dynamics perturbed}
    \dt \hat{f}_t(\vx) = -\eta \E_{\vx'} [ \Theta(\vx, \vx') \ell'(\hat{f}_t(\vx'); y(\vx')) ],
\end{equation}
until termination time $T$.

Under such a perturb-then-train process, we model the epistemic uncertainty $U$ as the difference between converged networks, reads $\| f_T(\vz) - \hat{f}_T(\vz) \|$, as illustrated in Fig.~\ref{fig: illustration-tulip}.
It measures the fluctuation of the training process, capturing the sensitivity of training w.r.t. noise.
Indeed, by applying a perturbation at $t = 0$, we essentially perturb $f_\mathrm{Init}$, which can be seen as a sampling process from some model prior 
\iffull
(Appendix~\ref{appendix: eq-1-eq-5-relation}).
\else
(See~\cite{tuliparxiv} for details).
\fi
In this case, $\hat{f}_T(\vz)$ can be interpreted as samples from the trained ensemble as in Eq.~\ref{eq: GT-uncertainty}, where their variance reflects epistemic uncertainty.

However, as stated above, in practice we only know the trained network $f_T$ at $t = T$.
It would be impractical to recover the full training trajectory, apply the perturbation at $t = t_s$ and then retrain the network.
Therefore, in the following, we will come up with a bound of $\| f_T(\vz) - \hat{f}_T(\vz) \|$ given the strength of the perturbation $\Delta f$, which can be evaluated at $\vz$ without actually retrain the network.
Thus, the perturbation is \emph{hypothetical}, as it has never been applied in our practice.

\subsection{Bounding Linearized Training Fluctuations}
\label{subsec: bounding-fluctuations}

To present the bound, we shall introduce the following assumptions:
\begin{enumerate}[label=A\arabic*., leftmargin=*]
    \item \textbf{(Boundedness)} For $t \in [0, T]$, $f(\vx)$, $\nabla_\vtheta f(\vx)$, $\ell$ and $\ell'$ stay bounded, uniformly on $\vx$.
    \item \textbf{(Smoothness)} Gradient $\ell'$ of loss function $\ell$ is Lipschitz continuous: $\forall x \in X;\;\| \ell'(\hat{y}; y(\vx)) - \ell'(\hat{y}'; y(\vx)) \| \leq L \|\hat{y} - \hat{y}'\|$.
    \item \textbf{(Perturbation)} The perturbation $\Delta f$ can be uniformly bounded by a constant $\alpha$, that is, for all $\vx$ (not limited to the support of training data), i.e.,  $\forall x \in \mathbb{R}^d;\;\| \Delta f(\vx) \| \leq \alpha$.
    \item \textbf{(Convergence)} Most importantly, for the original network trained via Eq.~\ref{eq: Linearized dynamics} and the perturbed network trained via Eq.~\ref{eq: Linearized dynamics perturbed}, we assume \emph{near-perfect convergence} on the training set $\vx$ at termination time $t = T$, i.e., 
    $\exists \beta \in \mathbb{R}, \forall x \in X;\;\|f_T(\vx) - \hat{f}_T(\vx)\| \leq \beta$.
    \item \textbf{(Closeness)} The test point $\vz$ is close to the training set $X$ in the sense that $\inf_{\vx \in X} \| \nabla_\vtheta f(\vz) - \nabla_\vtheta f(\vx) \|_\mathrm{F}^2
    \leq \tr \left( \Theta(\vz, \vz)\!+\!\E_\vx\left[\Theta(\vx, \vx)\right]\!-\!2\E_\vx\left[\left|\Theta(\vz, \vx)\right|\right]\right).$
\end{enumerate}

Under reasonable conditions, it has been shown both empirically~\cite{DBLP:conf/iclr/ZhangBHRV17} and theoretically~\cite{DBLP:conf/icml/DuLL0Z19} that overparameterized NNs trained via SGD is able to achieve near-zero training loss on almost arbitrary training sets. To nice loss functions as $\ell(\vy;\vy')=0$ implies $\vy=\vy'$, this implies A4.

We note that the closeness assumption A5 is weak. Intuitively, it assumes that for $\vx \in X$, $\Theta(\vz, \vx)$ contains many positive singular values ($\vz$ is close to $\vx$ in the sense of $\Theta$).
Fig.~\ref{fig: toy-datasets}(c) provides empirical justifications for this closeness assumption.

\cite{jacot2018neural} connected lazy NNs trained with mean square error (MSE) loss and kernel ridge regression.
Essentially, it hints that under such a setup, an NN embeds datapoints $\vx$ into gradients $\nabla_\vtheta f(\vx)$.
Indeed, for a general class of loss functions,
it is possible to show that:

\begin{theorem} \label{theorem: direct gronwall}
    Under assumptions A1-A4, for a network $f$ trained with Eq.~\ref{eq: Linearized dynamics} and a perturbed network $\hat{f}$ trained with Eq.~\ref{eq: Linearized dynamics perturbed}, the perturbation applied at time $t_s = T - \Delta T$ bounded by $\alpha$, we have
    \begin{equation}
        \label{eq: original uncertainty bound}
        \| f_T(\vz) - \hat{f}_T(\vz) \| \leq \inf_{\vx \in X} C \| \nabla_\vtheta f(\vz) - \nabla_\vtheta f(\vx) \|_\mathrm{F} + 2 \alpha + \beta,
    \end{equation}
    where $C = \frac{\alpha \eta \bar{\Theta}_X^{1/2}}{\lambda_{max}}\left( e^{(T - t_s)L\lambda_{max}} - 1 \right)$,
    $\bar{\Theta}_X^{1/2} := \left\| \nabla_\vtheta f(\vx) \right\|_{\X}$ is average gradient norm over training data, and $\lambda_{max} := \frac{1}{\sqrt{N}} \| \mG \|$ for a generalized Gram matrix $\emG_{i,j} := \| \Theta(x_i, x_j) \|$ of dataset $X = \{x_1, x_2, \dots, x_N\}$.
    
    Furthermore, under assumption A5, up to constants $J$ and $K$,
    \begin{equation}
        \label{eq: bound-practical}
        \| f_T(\vz)\! - \! \hat{f}_T(\vz) \| \lesssim J \big[\tr\left(\Theta(\vz, \vz)+\E_x[\Theta(\vx, \vx)]\right) - 2 K\left\|\nabla_\vtheta f_T(\vz)\left(\vtheta_T\!-\!\vtheta_{t_s}\right)\right\| \big]^\frac{1}{2}.%
    \end{equation}
\end{theorem}

\begin{proof}
    With an arbitrarily chosen pivot point $\xx$ from the training set, it is possible to bound $\| f(\vz) - f(\xx) \|$ and $\| \hat{f}(\vz) - \hat{f}(\xx) \|$ by bounding the fluctuations on the training set.
    Eq.~\ref{eq: original uncertainty bound} then follows from assumption A4.
    Combine A5, linearized training and H\"older's inequality, we derive Eq.~\ref{eq: bound-practical}.
    \iffull
    Please check Sec.~\ref{appendix: proof-3.1} for details.
    \else
    Please check~\cite{tuliparxiv} for details.
    \fi
\end{proof}

We see that the bound Eq.~\ref{eq: original uncertainty bound} on the training fluctuation is dominated by the distance from test point $\vz$ to the training set $\X$ in the ``embedding space'' of gradients. It is then possible to obtain a bound on such a distance using trained network parameters as in Eq.~\ref{eq: bound-practical}. A key insight here is that the gradient information during the training process has been accumulated into the trained weights.

\subsubsection{Network Ensemble} We close this section by the fact that
\begin{equation}
    \tr(\Var_{\Delta f}[
        \hat{f}_T(\vz)
    ]) \leq
    \E_{\Delta f}[
        \| \hat{f}_T(\vz) - f_T(\vz) \|^2
    ], \label{eq: variance-target}
\end{equation}
which can be then bounded by Eq.~\ref{eq: bound-practical}.
As we stated before, $\hat{f}_T(\vz)$ can be seen as samples from the trained ensemble as in Eq.~\ref{eq: GT-uncertainty}.
In practice, it is often beneficial to obtain such samples.
In the next section, we will present a heuristic method to estimate $\hat{f}_T(\vz)$ by matching variances.

\section{Implementation}
\label{sec: impl}
In this section, we present the key implementation strategies that enhance the practical effectiveness of our method, TULiP, summarized in Alg.~\ref{alg: TULiP}.
We elaborate on its design in the following subsections by referring to lines in Alg.~\ref{alg: TULiP}.

In contrast to the linearized network $f(\vx; \vtheta)$, let $\ef_t(\vx; \vtheta)$ denote a network trained empirically.
Intuitively, trajectories of $f_t(\vx;\vtheta)$ and $\ef_t(\vx;\vtheta)$ is similar when $\vtheta^\ast = \vtheta_\mathrm{Init}$ with a small learning-rate~\cite{lee2019wide,geiger2020disentangling}.
Under a post-hoc setting, as only converged models are available, we take $t_s = 0$ and substitute $\vtheta_{t_s}$ with $\E\left[\vtheta_0\right] = \mathbf{0}$ (or other mean specified by initialization schemes, e.g., $\E[\gamma_0] = \mathbf{1}$ in BatchNorm layers) in our implementation.
We also use the empirical NTK $\tilde{\Theta}^\mathrm{emp}_T$ at convergence time as an approximation for the kernel $\Theta$ used in Eq.~\ref{eq: Linearized dynamics}:
\begin{equation}
    \label{eq: Layer-wise scaling}
     \tilde{\Theta}^\mathrm{emp}_T(\vz, \vx) := \nabla_\vtheta \ef_T(\vz) \nabla_\vtheta \ef_T(\vx)^\top \approx \Theta(\vz, \vx),
\end{equation}

We first introduce how we estimate Eq.~\ref{eq: bound-practical} using $\ef_T$ at $t = T$.
Then, we introduce the construction of surrogate posterior samples that greatly enhance our method.

\begin{algorithm}[t]
\caption{TULiP for Classifiers.}\label{alg: TULiP}
\textbf{Input}: Input $z \in \mathbb{R}^d$, trained parameters $\vtheta_T$, network $\ef(z;\vtheta) : \mathbb{R}^{d} \rightarrow \mathbb{R}^o$ \\
\textbf{Parameter}: Calibrated parameters $J, \Theta_{XX}$; Perturbation strength $\epsilon, \delta$; Parameter $\lambda$; Number of posterior samples $M$ \\
\textbf{Output}: Uncertainty score $U$
\begin{algorithmic}[1]
\State $\vtheta_{t_s} \gets \E{\vtheta_0}$

\For{$i = 1, \dots, M$}%
    \State Sample $v_i \in \mathbb{R}^{|\vtheta|}$ from $\mathcal{N}(0, \epsilon^2 \mI)$
    \State $\tilde{f}^{raw}_{i}(\vz) \gets \ef(z; \vtheta_T + v_i)$
\EndFor

\State $\widetilde{\Theta}_{\mathrm{Tr}}(\vz, \vz) \gets \frac{1}{M} \sum_i \|
		 \tilde{f}^{raw}_{i}(\vz) - \ef(\vz; \vtheta_T)
		 \|^2$ \Comment{Estimation of $\Tr \Theta(\vz,\vz)$}
\State {$D \gets \sqrt{o} \| \vphantom{\widetilde{\Theta}} \ef(\vz; \vtheta_T + \epsilon \delta (\vtheta_T - \vtheta_{t_s}) ) - \ef(\vz; \vtheta_T)\|$}
\State {$S \gets J^2 \cdot \left( \widetilde{\Theta}_{\mathrm{Tr}}(\vz, \vz) + \Theta_{XX} - \lambda D \right)$} \Comment{Estimation of Eq.~\ref{eq: bound-practical} up to square root}
\State $\gamma \gets \sqrt{\max ( S, 0 ) / \widetilde{\Theta}_{\mathrm{Tr}}(\vz, \vz) }$

\For{$i = 1, \dots, M$} \Comment{Surrogate posterior samples}
    \State $\tilde{f}_{i}(\vz) \gets (1 - \gamma) \ef(z; \vtheta_T) + \gamma \tilde{f}^{raw}_i(\vz) \vphantom{\left(\tilde{f}\right)}$
\EndFor

\State $U \gets \mathbb{H}_y (\frac{1}{M} \sum_i \mathrm{softmax}(\tilde{f}_i(\vz)))$
\end{algorithmic}
\end{algorithm}

\subsection{Estimation of Jacobian (Line 2 - 8)}
Estimating gradients explicitly is both time and memory-consuming, especially for networks with large output dimensions.
Fortunately, Eq.~\ref{eq: bound-practical} only contains $\vz$-dependent terms that can be computed via perturbations on model weights $\vtheta$.
Namely, for $\| \nabla_\vtheta f_T(\vz) \left(\vtheta_T - \vtheta_{t_s}\right) \|$:
\begin{equation}
    \label{eq: D-approximation}
    \lim_{\delta \rightarrow 0} \frac{1}{\delta} \left( \ef(z; \vtheta_T + \delta \left(\vtheta_T - \vtheta_{t_s}\right)) - \ef(z; \vtheta_T) \right) \approx \nabla_\vtheta f_T(\vz) \left(\vtheta_T - \vtheta_{t_s}\right)
\end{equation}
is used in line 7 of Alg.~\ref{alg: TULiP}, corresponds to a deterministic perturbation.
The additional $\sqrt{o}$ scaling arises naturally in the derivation of Thm.~\ref{theorem: direct gronwall}.

For $\Tr \Theta(\vz, \vz)$, we could estimate its value with Hutchinson's Trace Estimator~\cite{avron2011randomized} (line 2-6), corresponds to a Gaussian perturbation (noise) on $\vtheta$.
\begin{proposition}    \label{prop: Ozz}
 Suppose that $\ef$ is $\gamma$-smooth w.r.t. $\vtheta$, i.e.,
 $$\|\nabla_\vtheta\ef(\vz;\vtheta)-\nabla_\vtheta\ef(\vz;\vtheta')\|_{\mathrm{F}}\leq \gamma\|\vtheta-\vtheta'\|. $$
 Let $\rvv$ be a random variable such that
 $\E_{\rvv}[\rvv] = \boldsymbol{0}, \E_{\rvv}[\rvv \rvv^\top] = \epsilon^2 \mI$ and 
 $\E_{\rvv}[\|\rvv\|^k ] \leq  C_k\epsilon^k$ for $k=3,4$, where $C_k$ is a constant depending on $k$ and the dimension of $\rvv$. 
 Then, under A1, it holds that 
    \begin{equation}
        \label{eq: Ozz-trace-estimator}
	 \lim_{\epsilon \rightarrow 0} \frac{1}{\epsilon^2} 
        \E_{\rvv}\left[ \|
		 \ef(\vz; \vtheta_T + \rvv) - \ef(\vz; \vtheta_T)
		 \|^2 \right]
	=
	\tr \left(
        \tilde{\Theta}^\mathrm{emp}_T(\vz, \vz)
	    \right). 
    \end{equation}
\end{proposition}
Note that the multi-dimensional normal distribution with mean zero and variance-covariance matrix 
$\epsilon^2 \mI$ agrees to the condition of $\rvv$. 
Prop.~\ref{prop: Ozz} and the approximation Eq.~\ref{eq: Layer-wise scaling}
ensures that 
$\Tr \Theta(\vz,\vz)$ 
is approximated by
$\epsilon^{-2}\E_{\rvv}[\|\ef(\vz; \vtheta_T + \rvv) - \ef(\vz; \vtheta_T)\|^2]$ with a small $\epsilon$. 
\begin{proof}
    \iffull
    Please check Sec.~\ref{appendix: proof-4.1} for details.
    \else
    Please check~\cite{tuliparxiv} for details.
    \fi
\end{proof}

From above, $\vz$-relavent terms in Eq.~\ref{eq: bound-practical} can be approximated while avoiding explicit computation of $\nabla_\vtheta f(\vz)$.
Specifically, in line 8, $S$ provides an estimation of the upper-bound Eq.~\ref{eq: bound-practical} up to $\E_\vx[\Theta(\vx, \vx)]$, square root and constants.
Here, the hyper-parameter $\lambda$ acts as a proxy to the constant $K$.
Such approximation is implemented by perturbations to $\vtheta$, thus compatible with mini-batching, enabling fast computation with $\mathcal{O}(M)$ forward passes.

\subsection{Calibration on Validation Dataset}
In practice, a validation dataset $X_\mathrm{Val}$ of ID data is often available to produce reliable results~\cite{zhang2023openoodv15}.
Such data is valuable in finding optimal $J$ and $\E_{\vx \sim P_X}[\Theta(\vx, \vx)]$.

$\E_{\vx \sim P_X}[\Theta(\vx, \vx)]$ can be estimated straight-forwardly by using the empirical expectation on $X_\mathrm{Val}$.
To find $J$, we first construct $\tilde{f}_i(\vz)$ on $\vz \in X_\mathrm{Val}$ following the step shown in Alg.~\ref{alg: TULiP} (line 1-11) with chosen $\epsilon, \lambda$ and $J = 1$.
$\tilde{f}_i(\vz)$ for varying $J$s can be easily obtained by controlling the scaling factor $\gamma$.
Therefore, we find the optimal non-negative $J^\ast$ by maximizing the likelihood of $\frac{1}{M} \Sigma_i \mathrm{softmax}(\tilde{f}_i(\vz))$ over $\vz \in X_\mathrm{Val}$.

For OOD detection however, we empirically found that a bigger $J$ sometimes gives better results.
Therefore, we set $J = J_\mathrm{scaling} \cdot J^\ast$ with hyper-parameter $J_\mathrm{scaling} \geq 1$ in practice.

\subsection{Surrogate Posterior Envelope (SPE) (Line 9 - 13)}
\label{subsec: surrogate-ensemble}
Many OOD detection methods work with the predictions produced by a neural network~\cite{yang2024generalized}.
As shown later in Table~\ref{table: OpenOOD with baselines}, a sufficiently trained model's raw prediction output (MSP) is already a robust OOD estimator.
Such OOD detection capability can be combined with Thm.~\ref{theorem: direct gronwall} to further improve TULiP's performance.

Intuitively, a well-trained classifier can be certain that a test datum $\vz$ belongs to neither class, in the sense that an evaluation of Eq.~\ref{eq: bound-practical} yields a small value, but its prediction $\mathrm{softmax}(f_T(\vz))$ indicates OOD input (i.e., it belongs to neither class).
To this end, borrowing ideas from the Model Averaging literature~\cite{gal2016dropout,DBLP:conf/nips/Lakshminarayanan17}, we propose to use the averaged prediction of the \emph{hypothetically-perturbed-then-trained} models $\E_{\Delta f} [\mathrm{softmax}(\hat{f}(\vz))]$ for OOD detection.

The distribution of $\hat{f}(\vx)$ is intractable without access to the training process.
Fortunately, we may approximate it via the bound Eq.~\ref{eq: bound-practical} by constructing an envelope around the true distribution, as depicted in Fig.~\ref{fig: illustration-tulip}.
From line 11, it is possible to show
\begin{equation}
\label{eq: pushed}
    \tr(\Var_i[\tilde{f}_i(\vz)]) \approx \gamma^2 \cdot \tr(\Var_i[\tilde{f}^{raw}_i(\vz)]) \; = S,
\end{equation}
for a positive $S$ and small $\epsilon$ such that $\E_i[\tilde{f}^{raw}_i(\vz)] \approx \ef(\vz; \vtheta_T)$.
Note that $\gamma$ is given in line 9 of Alg.~\ref{alg: TULiP}, and $S$ is an estimate of Eq.~\ref{eq: bound-practical} as stated in the previous subsection.
\iffull
Sec.~\ref{appendix: detail-4.3}
\else
\cite{tuliparxiv}
\fi
provides additional derivations to clarify their relationships.

For classification problems, after obtaining $\vv := \E_{\Delta f} [\mathrm{softmax}(\hat{f}(\vz))]$, TULiP applies the information entropy $\mathbb{H}_y [\vv]:= - \sum_{y=1}^o \vv_y \log \vv_y$ to produce OOD score as shown in line 13 of Alg.~\ref{alg: TULiP}.
Other methods, such as GEN~\cite{10203747}, can also be naturally incorporated into TULiP.

Alg.~\ref{alg: TULiP} summarizes TULiP, our proposed uncertainty estimator for OOD detection.
Although Alg.~\ref{alg: TULiP} gives TULiP for classification, it naturally generalizes to non-classification problems as TULiP constructs surrogate posterior samples.
In practice, TULiP can be accessed via: 
1) Obtain the trained model on ID dataset,
2) Calibrate $J$ and $\Theta_{XX}$ on a small validation dataset and
3) Apply Alg.~\ref{alg: TULiP} on test data.

\begin{figure}[t!]
\centering
\begin{adjustbox}{width=\textwidth}
\begin{subfigure}[b]{0.45\textwidth}
\centering
    \includegraphics{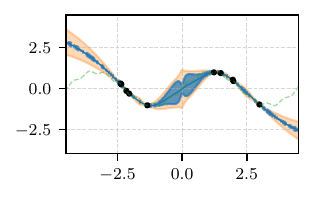}
    \caption{Regression on Toy dataset}
\end{subfigure}
\begin{subfigure}[b]{0.45\textwidth}
\centering
    \includegraphics{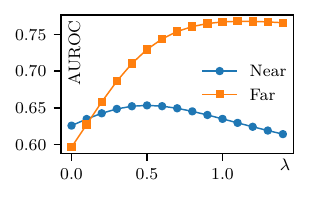}
    \caption{Raw performance of Eq.~\ref{eq: bound-practical}}
\end{subfigure}
\begin{subfigure}[b]{0.35\textwidth}
    \centering
    \includegraphics{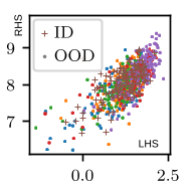}
    \caption{Verification of A5}
\end{subfigure}
\end{adjustbox}
\caption{Verification of Thm.~\ref{theorem: direct gronwall} (Sec.~\ref{subsec: empirical-validation}). From left to right, (a): Regression on Splines. Light shade: the bound Eq.~\ref{eq: original uncertainty bound}, heavy shade: Ground-truth ensemble (Eq.~\ref{eq: GT-uncertainty}), black dots: training data. (b): OOD Detection capability of Eq.~\ref{eq: bound-practical} on ImageNet-1K setup, using $S$ in Alg.~\ref{alg: TULiP} (line 13, used to estimate Eq.~\ref{eq: bound-practical}) as OOD scoring function. (c): Verification of A5 on ImageNet-1K setup.}
\label{fig: toy-datasets}
\end{figure}

\section{Experiments}

A key assumption of TULiP is linearized training.
However, it raises practical concerns, as empirical deep networks are typically trained non-linearly~\cite{NEURIPS2020_40507569}.
In this section, we demonstrate that TULiP nevertheless works under practical setups, by investigating TULiP in OOD tasks and comparing it with state-of-the-art OOD detectors.

\subsubsection{Experiment Setup}
We evaluate the practical performance of TULiP with OOD detection tasks based on manually defined ID-OOD dataset pairs, following the OpenOOD v1.5 benchmark~\cite{zhang2023openoodv15}.
OOD data range across a collection of diverse image datasets, categorized into \emph{near} and \emph{far} OOD sets~\cite{yang2022openood}, where near is more similar to ID and therefore more difficult to distinguish.
Following their setup, we use the same pre-trained ResNet-18~\cite{7780459} for CIFAR-10 \& 100~\cite{krizhevsky2009learning} and ImageNet-200~\cite{zhang2023openoodv15} ID datasets, and ResNet-50 for ImageNet-1K~\cite{imagenet15russakovsky}.

For TULiP, we use $M = 10$ surrogate posterior samples with $\epsilon = 0.005$, $\delta = 8$ and $\lambda = 1.25$.
We perturb all available parameters in the network including biases and Batch Normalization parameters $\beta, \gamma$~\cite{batchnorm}.
Following~\cite{zhang2023openoodv15}, we conduct a hyper-parameter search on a small validation set whenever possible, within a reasonable range of $J_\mathrm{scaling} \in \{1.0, 1.25, 1.5, 1.75, 2.0\}$.
\iffull
We explain our choice for hyper-parameters in Sec.~\ref{appendix: hparams}.
We present details of all datasets in Sec.~\ref{appendix: datasets}.
\else
\cite{tuliparxiv} explains the choice of hyper-parameters and details of all datasets.
\fi

\label{sec: experiments}
\subsection{Empirical Validation for Section~\ref{sec: method}}
\label{subsec: empirical-validation}

\subsubsection{Synthetic Datasets}
We begin by validating the original bound presented in Eq.~\ref{eq: original uncertainty bound} using toy regression data.
A 3-layer infinite-wide feed-forward neural network is used, and we solved the lazy training dynamics over the dataset using the \emph{neural-tangents} library~\cite{neuraltangents2020}.
We used MSE loss and computed the exact Gaussian ensemble~\cite{lee2019wide}.
Results are shown in Fig.~\ref{fig: toy-datasets}(a).
It suggests that our bound Eq.~\ref{eq: original uncertainty bound} based on training fluctuations is able to capture the true epistemic uncertainty as in Eq.~\ref{eq: GT-uncertainty}, justifies further developments of our method.

\subsubsection{Eq.~\ref{eq: bound-practical} in Practice}

\begin{figure}[t!]
\centering
\begin{subfigure}[b]{0.28\textwidth}
\centering
    \includegraphics[width=\textwidth]{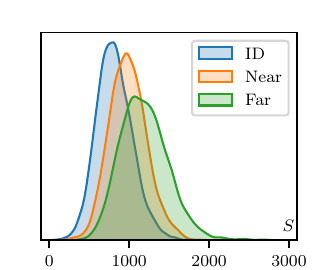}
    \caption{Eq.~\ref{eq: bound-practical} dist.}
\end{subfigure}
\begin{subfigure}[b]{0.35\textwidth}
\centering
    \includegraphics[width=\textwidth]{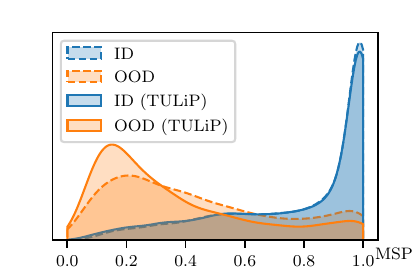}
    \caption{TULiP on MSP}
\end{subfigure}
\begin{subfigure}[b]{0.35\textwidth}
\centering
    \includegraphics[width=\textwidth]{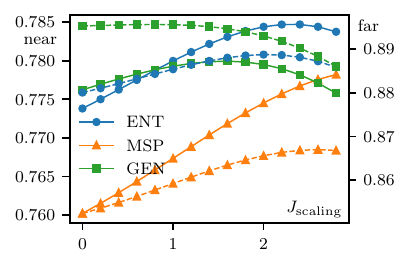}
    \caption{TULiP's enhancing effect}
\end{subfigure}
\caption{
(a) Visualization of $S$ in Alg.~\ref{alg: TULiP} (line 13, used to estimate Eq.~\ref{eq: bound-practical}) in various datasets (left to right: ID (ImageNet-1K), NINCO (near), Textures (far)).
(b) Visualization of the MSP score before and after applying TULiP on ID (right) and OOD (left). Notice how OOD data is pushed further towards 0 (lower confidence) than ID. We chose MSP due to its simplicity.
(c) Effect of TULiP on various OOD scoring criteria for near-OOD (solid) and far-OOD (dashed). $J_\mathrm{scaling}$ (horizontal axis) controls the strength of TULiP, where a value of $0$ indicates baselines without TULiP.
}
\label{fig: tulip-dev}
\end{figure}

We further validate the proposed bound Eq.~\ref{eq: bound-practical} by its OOD detection performance in practical scenarios.
Using the ImageNet-1K setup, Fig.~\ref{fig: toy-datasets}(b) shows the positive OOD detection capability of Eq.~\ref{eq: bound-practical}.
In practice, it is estimated using $S$ in Alg.~\ref{alg: TULiP} (line 13).

By setting $\lambda = 0$, we effectively recover an OOD scoring criteria similar to GradNorm~\cite{GradNorm_NeurIPS}.
GradNorm computes the OOD score based on the gradient of KL divergence between softmax-ed prediction and a uniform distribution, where ID data is claimed to have larger gradient compared to OOD data.
In contrast, our OOD score in Fig.~\ref{fig: toy-datasets}(b) (when $\lambda = 0$), is computed from the Frobenius norm of gradient of \emph{logits} w.r.t. network parameters evaluated at given input.
In logit space, it is shown that the OOD data instead have a slight \emph{larger} gradient compared to the ID data (Fig.~\ref{fig: toy-datasets}(b), when $\lambda = 0$, AUROC $> 0.5$ indicates that the OOD data have larger gradients on average).

Furthermore, the performance of Eq.~\ref{eq: bound-practical} is boosted when using a positive $\lambda$, supports our claim that the parameter-gradient product $\|\nabla_\vtheta f_T(\vz) (\vtheta_T - \vtheta_{t_s})\|$ provides useful information of training dataset as Thm.~\ref{theorem: direct gronwall} suggested.
Especially in the \emph{far-OOD} split, such performance boost of the $\lambda$-term is surprisingly significant (improved AUROC over 15\%).
\emph{Far-OOD} contains images that are different from ID more significantly.
It could potentially be an easy case for Eq.~\ref{eq: bound-practical}, in the sense that a relatively looser bound (as having no access to training data), especially $\|\nabla_\vtheta f_T(\vz) (\vtheta_T - \vtheta_{t_s})\|$, can still retrieve enough training information for OOD detection, as discussed in Sec.~\ref{subsec: bounding-fluctuations}.
Note that despite being far, \emph{far-OOD} data still satisfy the closeness assumption A5 as shown in Fig.~\ref{fig: toy-datasets}(c).

\subsection{Out-of-distribution Detection}
\label{subsec: main-experiment}

\subsubsection{Baseline Methods} 
Shannon Entropy (ENT) computes the entropy of plain network prediction.
ENT is an important baseline, as TULiP works by providing it with TULiP-enhanced predictions (average from SPE).
We also consider various baselines other than ENT for comparison, including the MC-Dropout (MCD), post-hoc OOD methods without training data ODIN, EBO, MLS, ASH and GEN; and finally, MDS and ViM with access to training data.
Please refer to Sec.~\ref{sec: related} for a brief introduction.

\subsubsection{Results}
We report the performance of TULiP on OpenOOD v1.5 benchmark~\cite{zhang2023openoodv15} in Table~\ref{table: OpenOOD with baselines}.
TULiP consistently improves over the ENT on both near- and far-OOD setups.
It achieves remarkable performance in near-OOD settings with either top-1 or top-2 AUROC scores across all datasets, and is on par with most methods in far-OOD settings.
It is essential to note that methods that significantly outperform TULiP on far-OOD either access the training dataset (ViM and MDS) or completely lack a theoretical explanation (ASH).
As we discuss below, TULiP is also more consistent compared to other methods such as ASH.

\begin{table}[t!]
\centering
\caption{Results on OpenOOD benchmark, averaged from 3 runs. The top results for each category are marked in bold, with the second-best result in underlined. We include baseline results from~\cite{zhang2023openoodv15}, and reproduced the results for MC-Dropout (MCD). A dagger symbol $\dag$ indicates direct access to training data or processes.
Results are averaged separately for \emph{near} / \emph{far}-OOD sets.
}
\label{table: OpenOOD with baselines}
\begin{adjustbox}{width=\textwidth}
\begin{tabular}{l c c c c c c c c} 
\toprule
    & \multicolumn{2}{c}{CIFAR-10}      & \multicolumn{2}{c}{CIFAR-100}     & \multicolumn{2}{c}{ImageNet-200}  & \multicolumn{2}{c}{ImageNet-1K}   \\ [0.2ex]
      Method                    & FPR@95 $\downarrow$ & AUROC $\uparrow$ & FPR@95 $\downarrow$ & AUROC $\uparrow$ & FPR@95 $\downarrow$ & AUROC $\uparrow$ & FPR@95 $\downarrow$ & AUROC $\uparrow$ \\
      \midrule 
      MCD $\dag$              &53.54/31.43&87.68/91.00&\underline{54.73}/59.08&80.42/77.58&55.25/35.48&{83.30}/90.20&65.68/51.45&76.02/85.23\Tstrut\\
      MDS $\dag$              &49.90/32.22&84.20/89.72&83.53/72.26&58.69/69.39&79.11/61.66&61.93/74.72&85.45/62.92&55.44/74.25 \\
      ViM $\dag$              &\underline{44.84}/\textbf{25.05}&\underline{88.68}/\textbf{93.48}&62.63/\textbf{50.74}&74.98/\textbf{81.70}&59.19/\textbf{27.20}&78.68/91.26&71.35/\underline{24.67}&72.08/92.68 \\ 
      \midrule
      MSP               &48.17/31.72&88.03/90.73&54.80/58.70&80.27/77.76&\underline{54.82}/35.43&83.34/90.13&65.68/51.45&76.02/85.23 \\
      ODIN              &76.19/57.62&82.87/87.96&57.91/58.86&79.90/79.28&66.76/34.23&80.27/91.71&72.50/43.96&74.75/89.47 \\
      MLS               &61.32/41.68&87.52/91.10&55.47/56.73&\underline{81.05}/79.67&59.76/34.03&82.90/91.11&\textbf{51.35}/63.60&76.46/89.57 \\
      ASH               &86.78/79.03&75.27/78.49&65.71/59.20&78.20/\underline{80.58}&64.89/\underline{27.29}&82.38/\textbf{93.90}&\underline{63.32}/\textbf{19.49}&\underline{78.17}/\textbf{95.74} \\
      GEN               &53.67/47.03&88.20/91.35&\textbf{54.42}/56.71&\textbf{81.31}/79.68&55.20/32.10&\underline{83.68}/91.36&65.32/35.61&76.85/89.76 \\
      EBO               &61.34/41.69&87.58/91.21&55.62/56.59&80.91/79.77&60.24/34.86&82.50/90.86&68.56/38.39&75.89/89.47 \\
      ENT               &48.17/31.71&88.12/90.84&54.81/58.69&80.63/78.14&\underline{54.82}/35.34&83.55/90.52&65.63/51.28&76.55/86.21 \\
      \midrule
      ReAct (EBO)       &63.56/44.90&87.11/90.42&56.39/\underline{54.20}&80.77/80.39&62.49/28.50&81.87/\underline{92.31}&66.69/26.31&77.38/\underline{93.67} \\
      \rowcolor{Gainsboro}
      TULiP (ENT)             &\textbf{40.77}/\underline{28.40}&\textbf{89.36}/\underline{92.25}&56.21/57.94&80.81/78.75&\textbf{54.75}/33.62&\textbf{83.84}/91.05&64.96/48.01&\textbf{78.38}/88.85 \\ 
      \bottomrule
\end{tabular}
\end{adjustbox}
\end{table}

\begin{table}[t]
    \centering
    \scriptsize
    \caption{Extended results on ImageNet-1K (in AUROC $\uparrow$). Some baseline results cited from~\cite{zhang2023openoodv15}. Evaluating ASH on OpenOOD framework~\cite{yang2022openood} with unsupported network architectures (MobileNet, VGG-16, RegNet) is difficult and thus omitted.}
    \label{table: extended-results}
    \begin{tabular}{l c c c c c} 
    \toprule
        Method    &ResNet-50(V2)& ViT-16-B         & MobileNet-V3-Large   & VGG-16      & RegNet-Y-16GF      \\
    \midrule
        EBO       & 54.39/51.36 & 62.41/78.98 & 71.00/77.46 & 72.16/89.92 & 67.50/69.98 \\
        ODIN      & 69.89/66.10 & 64.32/76.06 & 70.09/79.63 & 70.86/90.83 & 64.93/65.41 \\
        MLS       & 69.24/76.47 & 68.30/83.54 & 74.39/82.67 & 72.75/89.78 & 72.07/78.45 \\
        ASH       & 35.96/29.91 & 53.21/51.56 &     -/-     &     -/-     &     -/-     \\
        TULiP     & 66.77/74.35 & 72.96/87.42 & 75.20/84.40 & 74.54/87.51 & 74.54/83.52 \\
    \bottomrule
    \end{tabular}
\end{table}

\subsubsection{Consistency}
To further validate the consistency of TULiP, we conduct ImageNet-1K OOD experiments with various network architectures and training protocol, using pretrained models available online~\cite{torchvision2016}.
We report the results in Table~\ref{table: extended-results}.
TULiP consistently achieves top-class performance across various setups.
Notably, ASH fails when ResNet-50(V2) weights are used.
V2 weights are trained with recent advances in practical NN training, increasing the Top-1 accuracy by 4.7\%~\cite{IMAGENET1K_V2}.
Such a complicated training process may render difficulties when applying methods like ASH without theoretical guidance.

\subsubsection{TULiP Enhancing Baseline Methods}
The proposed algorithm Alg.~\ref{alg: TULiP} works with Shannon Entropy to compute an OOD score from averaged predictions.
Such choice is due to the superior performance of the entropy score when paired with TULiP, as shown in Table~\ref{table: OpenOOD with baselines}.
Indeed, other methods can also work with TULiP's averaged prediction, as TULiP enhances their performance on OOD tasks.
We demonstrate such an enhancing effect using MSP score (i.e., raw confidence of predicted class) in Fig.~\ref{fig: tulip-dev}(b), using ImageNet-1K (ID) and Textures (OOD).
With weight perturbations, TULiP pushes original network predictions towards a more uniform distribution, with an increased effect particularly on OOD data.
We compare the enhancing effect of various methods in Fig.~\ref{fig: tulip-dev}(c).

ReAct~\cite{ReAct_DBLP:journals/corr/abs-2111-12797} is similar to TULiP as an enhancement method (in this case, over EBO), while TULiP provides a more consistent improvement in all datasets.
TULiP can also be combined with ReAct as they operate on different aspects of a network, yet the details of such a combination need further investigation.

\begin{figure}[t!]
\centering
\begin{adjustbox}{width=0.9\textwidth}
\includegraphics{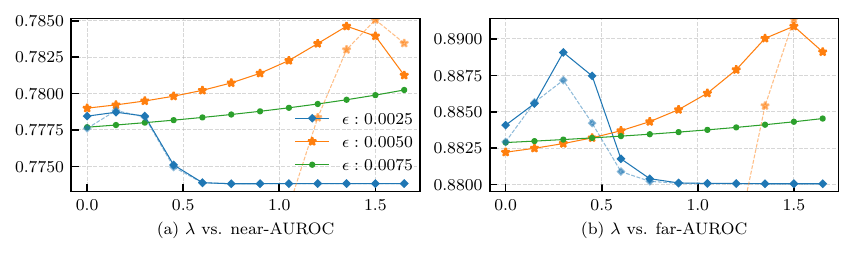}
\end{adjustbox}
\caption{Results by varying $\epsilon$ and $\lambda$ on ImageNet-1K ID. Solid lines: A validation set is used to find the the optimal $J$ with $J_\mathrm{scaling} = 2.0$. Dashed lines: No validation set calibration for $J$ (results for $\epsilon = 0.0075$ falls out of the range of the plot).}
\label{fig: ablation-hparam}
\end{figure}

\subsubsection{Ablation Study and Hyper-parameters}
We conduct experiments on ImageNet-1K to analyze the effect and consistency of hyper-parameters.
The results are shown in Fig.~\ref{fig: ablation-hparam}, where we observe a consistent behavior of the parameter $\lambda$ in near- and far-OOD setups.
Given a fixed $\epsilon$, the optimal $\lambda$ for near- and far-OOD setups is similar, simplifies the hyper-paremter tuning process.
It also further validates the positive effect of the $\lambda$-term in TULiP for OOD detection.
As shown in Fig.~\ref{fig: ablation-hparam}, a slight larger $\epsilon$ can also increase the stability of TULiP w.r.t. changes in $\lambda$, where larger $\epsilon$ favors larger $\lambda$ (also corresponds to smaller $J^\ast$).
In practice, the calibration step of finding an optimal $J$ ensures consistent performance across various setups, as shown in Fig.~\ref{fig: ablation-hparam} with dashed lines.

\section{Conclusion}

In this study, we present TULiP, an uncertainty estimator for OOD detection.
Our method is driven by the fluctuations under linearized training dynamics and excels in practical experiments.
However, there are some limitations and future works remaining.
Theoretically, our framework only considers functional perturbation. The perturbation on the NTK is also important~\cite{kobayashi2022disentangling} and could be integrated into the estimator in the future.
Future works may improve upon these aspects, covering a wider range of OOD data by examining the network parameters and refining weight perturbations. 
In a broader aspect, exploring TULiP in other learning paradigms, such as Active Learning~\cite{NEURIPS2022_a102dd59} or Reinforcement Learning~\cite{10.5555/1855083} will be valuable.

\begin{credits}
\subsubsection{\ackname} 
This work was partially supported by JSPS KAKENHI Grant Numbers 20H00576, 23H03460, and 24K14849.

\subsubsection{\discintname}
The authors have no competing interests to declare that are relevant to the content of this article.

\end{credits}

{
    \section*{\refname}
    \small
    \let\clearpage\relax
    \renewcommand{\section}[2]{}
    \renewcommand{\chapter}[2]{}
    \bibliographystyle{splncs04}
    \bibliography{ref_tidy_abbrv}
}

\appendix

\section{Proofs}
\subsection{Basic notations}
For a network $f(\vx): \mathbb{R}^d \rightarrow \mathbb{R}^o$ maps inputs $\vx$ of dimension $d$ to outputs $f(\vx)$ of dimension $o$, parameterized by $\vtheta$ with $|\vtheta|$ trainable parameters, the gradient / Jacobian matrix $\nabla_\vtheta f(\vx)$ is a $o \times |\vtheta|$ matrix. 

The NTK $\Theta(\vz, \vx) := \nabla_\vtheta f(\vz) \nabla_\vtheta f(\vx)^\top$ is a $o \times o$ matrix.

$\ell' (f_t(\vx))$ is the gradient of loss function w.r.t. network output $f_t(\vx)$ at training time $t$. It is, for convenience, a $o \times 1$ column-vector.

The following lemma will be useful thereafter, which is an application of H\"older's inequality.

\begin{lemma}
    Let $F: x \rightarrow \mathbb{R}^{m \times n}$, $g: x \rightarrow \mathbb{R}^n$. Consider 2-norms $\| \cdot \|$ (i.e., euclidean and its induced matrix 2-norm). For $p, q \in [1, \infty]$ that $\frac{1}{p} + \frac{1}{q} = 1$, we have
    \begin{align*} 
        & \; \left\|\E_x[F(\vx) g(\vx)]\right\| \\ 
        \leq & \; \E_x[\|F(\vx) g(\vx)\|] \\
        \leq & \; \E_\vx\left[\|F(\vx)\| \cdot \|g(\vx)\|\right] \\
        \leq & \; \E_\vx\left[\|F(\vx)\|^p\right]^{1 / p} \cdot \E_\vx\left[\|g(\vx)\|^q\right]^{1 / q}.
    \end{align*}
\end{lemma}
When $q = \infty$, we have $\E_\vx\left[\left\| g(\vx) \right\|^q\right]^{1/q} := \sup_x \left\| g(\vx) \right\|$.

For convenience, given any random variable, vector or matrix $\mathbf{A}$ dependent of $\vx$, we denote:
\begin{equation}
    \left\| \mathbf{A} \right\|^{(q)}_{\X} := \E_{\vx} \left[ \|\mathbf{A}\|^q \right]^{1/q},
\end{equation}
which by itself is a valid norm.
We omit superscript $(q)$ if $q = 2$.

\subsection{Assumptions}
We recall the assumptions here, which are originally shown in Sec.~\ref{sec: method}.
For network $f(\vx, \vtheta)$, dataset $X$ with no parallel datapoints and a twice-differentiable loss function $\ell$, we assume the followings:
\begin{enumerate}[label=A\arabic*., leftmargin=*]
    \item (Boundedness) For $t \in [0, T]$, $f(\vx)$, $\nabla_\vtheta f(\vx)$, $\ell$ and $\ell'$ stay bounded, uniformly on $\vx$.
    \item (Smoothness) Gradient $\ell'$ of loss function $\ell$ is Lipschitz continuous: $\forall x \in X;\;\| \ell'(\hat{y}; y(\vx)) - \ell'(\hat{y}'; y(\vx)) \| \leq L \|\hat{y} - \hat{y}'\|$.
    \item (Perturbation) The perturbation $\Delta f$ can be uniformly bounded by a constant $\alpha$, that is, for all $\vx$ (not limited to the support of training data), i.e.,  $\forall x \in \mathbb{R}^d;\;\| \Delta f(\vx) \| \leq \alpha$.
    \item (Convergence) Finally, for the original network trained via~\eqref{eq: Linearized dynamics} and the perturbed network trained via~\eqref{eq: Linearized dynamics perturbed}, we assume \emph{near-perfect convergence} on the training set $\vx$ at termination time $t = T$, i.e., 
    $\exists \beta \in \mathbb{R}, \forall x \in X;\;\|f_T(\vx) - \hat{f}_T(\vx)\| \leq \beta$.
\end{enumerate}

\subsection{Proof of Equation 5}
\label{appendix: proof-3.1}

We proof Thm.~\ref{theorem: direct gronwall} (\eqref{eq: original uncertainty bound} and \eqref{eq: bound-practical}) separately.

\begin{theorem} \label{theorem: direct gronwall replicated}
    (Equation~\ref{eq: original uncertainty bound})
    Under assumptions A1-A4, for a network $f$ trained with~\eqref{eq: Linearized dynamics} and a perturbed network $\hat{f}$ trained with~\eqref{eq: Linearized dynamics perturbed}, the perturbation applied at time $t_s = T - \Delta T$ bounded by $\alpha$, we have
    \begin{equation}
        \label{eq: original uncertainty bound replicated}
        \| f_T(\vz) - \hat{f}_T(\vz) \| \leq \inf_{\vx \in X} C \| \nabla_\vtheta f(\vz) - \nabla_\vtheta f(\vx) \|_\mathrm{F} + 2 \alpha + \beta,
    \end{equation}
    where $C = \frac{\alpha \eta \bar{\Theta}_X^{1/2}}{\lambda_{max}}\left( e^{(T - t_s)L\lambda_{max}} - 1 \right)$,
    $\bar{\Theta}_X^{1/2} := \left\| \nabla_\vtheta f(\vx) \right\|_{\X}$ is the average gradient norm over training data, and $\lambda_{max} := \frac{1}{\sqrt{N}} \| \mG \|$ for a generalized Gram matrix $\emG_{i,j} := \| \Theta(x_i, x_j) \|$ of dataset $X = \{x_1, x_2, \dots, x_N\}$.
\end{theorem}

\begin{proof}
Let us first examine the fluctuations in the training set.
From the Lipschitz continuity of $\ell'$,
\begin{align}
    \left\|\ell'(f_t(\vx))-\ell'(\hat{f}_t(\vx))\right\|_{\X}
    \leq L \| f(\vx) - \hat{f}(\vx) \|_{\X}.
\end{align}

Thus, by the linearized dynamics we have
\begin{align}
    & \partial_t \left\|f_t(\vx) - \hat{f}_t(\vx)\right\|_{\X} \nonumber \\
    \leq & \left\| \partial_t \left( f_t(\vx) - \hat{f}_t(\vx) \right) \right\|_{\X} \label{eq: triangle-norm-dt} \\
    = & \left\| \E_{x'}\left[\Theta(\vx, \vx')\left(\ell'(f_t(\vx'))-\ell'(\hat{f}_t(\vx'))\right)\right] \right\|_{\X} \nonumber \\
    = & \E_\vx\left[
        \left\|
            \E_{x'}\left[\Theta(\vx, \vx')\left(\ell'(f_t(\vx'))-\ell'(\hat{f}_t(\vx'))\right)\right]
        \right\|^2
    \right]^{1/2} \nonumber \\
    \leq & \E_x \left[
        \left\|
            \Theta(\vx, \vx')
        \right\|_{\X}^2
        \left\|
            \ell'(f_t(\vx'))-\ell'(\hat{f}_t(\vx'))
        \right\|_{\X}^2
    \right]^{1/2} \nonumber \\
    \leq & \E_{x, x'} \left[ \left\| \Theta(\vx, \vx') \right\|^2 \right]^{1/2}
    \left\|
        \ell'(f_t(\vx))-\ell'(\hat{f}_t(\vx))
    \right\|_{\X} \nonumber \\
    \leq & L \lambda_{max} 
    \left\|
        f_t(\vx)-\hat{f}_t(\vx)
    \right\|_{\X}, \label{eq: train-set-fluc-dt-bound}
\end{align}
where in~\eqref{eq: triangle-norm-dt} we have used the triangle inequality to put $\partial_t$ inside the norm.
$\lambda_{max}$ is defined as $\frac{1}{\sqrt{N}} \| \textbf{G} \|$ for a generalized Gram-matrix $\textbf{G}_{ij} := \| \Theta(x_i, x_j) \|$ of dataset $X = \{x_1, x_2, \dots, x_N\}$, measures the fitness (or alignment) of the kernel $\Theta$ w.r.t. the training data.

From~\eqref{eq: train-set-fluc-dt-bound}, we can apply the Gr\"onwall's inequality to obtain
\begin{align}
    & \left\|f_t(\vx)-\hat{f}_t(\vx)\right\|_{\X} \nonumber \\
    \leq & \left\|f_{t_s}(\vx)-\hat{f}_{t_s}(\vx)\right\|_{\X} e^{(t - t_s)L\lambda_{max}} \nonumber \\
    \leq & \; \alpha e^{(t - t_s)L\lambda_{max}}. \label{eq: train-set-fluc-bound}
\end{align}
We now prove Theorem 3.1 by generalizing~\eqref{eq: train-set-fluc-bound} to given test data.

For a test point $z \in \mathbb{R}^d$, choose a pivot point $\xx \in X$ from the training set.
Then for the network function $f$ evaluated at $\xx$ and $\vz$, we have the followings:

\begin{align}
    & \left| \partial_t \left\| \left(f_t(\vz)-f_t\left(\xx\right)\right)-\left(\hat{f_t}(\vz)-\hat{f_t}\left(\xx\right)\right) \right\| \right| \nonumber \\
    \leq & \left\| \partial_t\left[\left(f_t(\vz)-f_t\left(\xx\right)\right)-\left(\hat{f_t}(\vz)-\hat{f_t}\left(\xx\right)\right)\right] \right\| \nonumber \\
    = & \eta \left\| \E_\vx\left[\left(\Theta(\vz, \vx)-\Theta\left(\vx^\ast, \vx\right)\right) \cdot\left(\ell'(f_t(\vx))-\ell'(\hat{f}_t(\vx))\right)\right]\right\| \label{eq: 3.1-1}
\end{align}

Denote 
$$ \left\| \left(f_t(\vz)-f_t\left(\xx\right)\right)-\left(\hat{f_t}(\vz)-\hat{f_t}\left(\xx\right)\right) \right\| $$
as $\Delta f_t(\vz)$, and let $\Theta^{\mathrm{diff}}_{\vx^\ast}(\vz, \vx) := \left(\Theta(\vz, \vx)-\Theta\left(\vx^\ast, \vx\right)\right)$.
Integrate~\eqref{eq: 3.1-1} with $t$, we have
\begin{align}
& \left|\Delta f_T(\vz)-\Delta f_{t s}(\vz)\right| \nonumber \\
\leq & \; \eta \int_{t_s}^T\left\|\E_\vx\left[\Theta^{\mathrm{diff}}_{\vx^\ast}(\vz, \vx)\left(\ell'(f_t(\vx))-\ell'(\hat{f}_t(\vx))\right)\right]\right\| d t \nonumber \\
\leq & \; \eta \int_{t_s}^T
    \left\|
        \Theta^{\mathrm{diff}}_{\vx^\ast}(\vz, \vx)
    \right\|_{\X}
    \left\|\ell'(f_t(\vx))-\ell'(\hat{f}_t(\vx))\right\|_{\X} d t \nonumber \\
\leq & \; \eta L  
    \left\|
        \Theta^{\mathrm{diff}}_{\vx^\ast}(\vz, \vx)
    \right\|_{\X}
    \int_{t_s}^T
    \left\|f_t(\vx)-\hat{f}_t(\vx)\right\|_{\X} d t \label{eq: 3-1-3}
\end{align}

We start with the term before the integral.
To begin, rewrite it as:
\begin{align}
    & \left\|
        \Theta(\vz, \vx) - \Theta(\vx^\ast, \vx)
    \right\|_{\X} \nonumber \\
 =  & \left\|
        \left(
            \nabla_\vtheta f(\vz) - \nabla_\vtheta f(\vx^\ast)
        \right)
        \nabla_\vtheta f(\vx)^\top
    \right\|_{\X} \nonumber \\
\leq& \E_x \left[
        \left\| \nabla_\vtheta f(\vz) - \nabla_\vtheta f(\xx) \right\|^2 \cdot \left\| \nabla_\vtheta f(\vx) \right\|^2
    \right]^{1/2} \nonumber \\
 =  & \left\| \nabla_\vtheta f(\vz) - \nabla_\vtheta f(\xx) \right\| \cdot \left\| \nabla_\vtheta f(\vx) \right\|_{\X} \nonumber \\
\leq& \left\|
        \nabla_\vtheta f(\vz) - \nabla_\vtheta f(\vx^\ast)
    \right\|_\mathrm{F} \cdot \bar{\Theta}_X^{1/2}, \label{eq: bound-dominant}
\end{align}

where $\bar{\Theta}_X^{1/2} := \left\| \nabla_\vtheta f(\vx) \right\|_{\X}$ is independent from $\vz$.

\begin{remark}
\eqref{eq: bound-dominant} used a computationally friendly Frobenius norm to bound the spectral norm in the line right above it. This is the main motivation to use a $\sqrt{o}$ scaling for hyper-parameter $\lambda$, as $\| \mA \|_2 \leq \| \mA \|_F \leq \sqrt{o} \| \mA \|_2$ given $\mA$ is full-rank $o \times | \vtheta |$ and $o < |\vtheta|$.
\end{remark}

Bring~\eqref{eq: bound-dominant} and~\eqref{eq: train-set-fluc-bound} back to~\eqref{eq: 3-1-3} we have
\begin{align}
& \left|\Delta f_T(\vz)-\Delta f_{t_s}(\vz)\right| \nonumber \\
\leq & 
    \left\|
        \nabla_\vtheta f(\vz) - \nabla_\vtheta f(\vx^\ast)
    \right\|_\mathrm{F}
    \frac{\alpha \eta \bar{\Theta}_X^{1/2}}{\lambda_{max}}\left( e^{(T - t_s)L\lambda_{max}} - 1 \right).
\end{align}
Finally, we bound the difference between $f_T(\vz)$ and $\hat{f}_T(\vz)$ via the triangle inequality.

First, observe that from A3,
\begin{align}
    \Delta f_{t_s}(\vz) = &
    \left\| \left(f_{t_s}(\vz)-f_{t_s}\left(\xx\right)\right)-\left(\hat{f_{t_s}}(\vz)-\hat{f_{t_s}}\left(\xx\right)\right) \right\| \nonumber \\
    \leq &
    \left\| f_{t_s}(\vz)-\hat{f}_{t_s}(\vz) \right\| + 
    \left\| \hat{f}_{t_s}(\vx^\ast)-f_{t_s}(\vx^\ast) \right\| \nonumber \\
    \leq & \;
    2 \alpha,
\end{align}
so that
\begin{align}
    \Delta f_T(\vz) \leq &
    \left|\Delta f_T(\vz)-\Delta f_{t_s}(\vz)\right| + \left|\Delta f_{t_s}(\vz)\right| \nonumber \\
    \leq & \left\|
        \nabla_\vtheta f(\vz) - \nabla_\vtheta f(\vx^\ast)
    \right\|_\mathrm{F}
    \frac{\alpha \eta \bar{\Theta}_X^{1/2}}{\lambda_{max}}\left( e^{(T - t_s)L\lambda_{max}} - 1 \right) + 2\alpha.
\end{align}
Thus, given the convergence assumption $\| f_T(\xx) - \hat{f}_T(\xx) \| \leq \beta$ (A4),
\begin{align}
    & \left\| f_T(\vz)-\hat{f}_T(\vz) \right\| \nonumber \\
    =& \left\| \left(\vphantom{\hat{f}_T} f_T(\vz) - f_T(\vx^\ast)\right) - \left(\hat{f}_T(\vz) - \hat{f}_T(\vx^\ast)\right) - \left(\hat{f}_T(\vx^\ast) - f_T(\vx^\ast)\right) \right\| \nonumber \\
    \leq& \; \Delta f_T(\vz) + \beta.
\end{align}

To proceed, recall that $\xx$ is chosen arbitrarily.
This concludes the proof.
\end{proof}

We can further write the equation as:
\begin{equation}
\inf_{\vx \in X} \left\| \nabla_\vtheta f(\vz) - \nabla_\vtheta f(\vx) \right\|_\mathrm{F}
=\inf_{\vx \in X} \big[ \tr\left( \Theta(\vz, \vz) + \Theta(\vx, \vx) - 2\Theta(\vz, \vx) \right) \big]^{1/2}. \label{eq: 3-NTKs}
\end{equation}
\subsubsection{Note for~\eqref{eq: 3-NTKs}:} Let $\mA := \nabla_\vtheta f(\vz) - \nabla_\vtheta f(\vx)$. Then, we have $\Tr \left( \mA \mA^\top \right) = \Tr \left( \Theta(\vz, \vz) + \Theta(\vx, \vx) - \Theta(\vz, \vx) - \Theta(\vx, \vz) \right)$. Note that $\Theta(\vx, \vz) = \Theta(\vz, \vx)^\top$ and therefore we may substitute them inside the trace. Thereafter, using $\| \mA \|_\mathrm{F} = \left( \Tr (\mA \mA^\top) \right)^{1/2}$, we can obtain~\eqref{eq: 3-NTKs}.

\subsection{Proof of Equation~\ref{eq: bound-practical}}
\label{appendix: proof-3.2}

 We first present the following lemma under the weakly lazy regime, i.e., we allow the weak dependency of $\Theta_t$ on $t$. 
 Let us define $\left| \Theta_T(\vz, \vx) \right|$ as the unique symmetric positive semi-definite solution of 
 $\left| \Theta_T(\vz, \vx) \right|^2 = \Theta_T(\vz, \vx)^\top \Theta_T(\vz, \vx)$, 
 which is an extension of absolute values to matrices. 
\begin{lemma}
 \label{lemma: NTK LB replica} %
 We assume the lazy learning regime, i.e., 
 there exists $\delta>0$ such that $\sup_{\vx,\vx'}\||\Theta_T(\vx,\vx')|-|\Theta_t(\vx,\vx')|\|\leq \delta$ holds 
 for all $t_s\leq t\leq T$. 
 Under assumption A1, with the model parameters $\vtheta_T$ trained from $\vtheta_{t_s}$ with~\eqref{eq: Linearized dynamics} over the training set $\vx$ and $t_s < T$, we have:
    \begin{equation}
        \label{eq: lemma NTK LB result replica}
	 \|\nabla_\vtheta f_T(\vz) (\vtheta_T - \vtheta_{t_s})\|
	\leq
	C(\tr\E_\vx [|\Theta_T(\vz, \vx)|] + o\delta)
	+\sqrt{\delta} \|\vtheta_T - \vtheta_{t_s}\|
    \end{equation}
 where $C$ is a positive constant independent of $\vz$. 
\end{lemma}
\begin{proof}
 The mean value theorem for integrals guarantees that there exists $\tau\in[t_s,T]$ such that 
 \begin{equation}
    \label{eq: param dynamics integrated}
     \vtheta_T - \vtheta_{t_s} = - \int_{t_s}^T \eta \E_\vx\left[ \nabla_\vtheta f_\tau(\vx) \ell'(f_{\tau}(\vx)) \right] dt.
\end{equation}
Then, H\"older's inequality leads to %
\begin{align}
 \left\|\nabla_\vtheta f_\tau(\vz) (\vtheta_T - \vtheta_{t_s})\right\| \notag 
  = &\left\| \E_\vx\left[ \nabla_\vtheta f_\tau(\vz) \nabla_\vtheta f_\tau(\vx)^\top \eta \int_{t_s}^T \ell'(f_{\tau}(\vx)) dt \right] \right\| \notag \\
\leq&\left\|\Theta_\tau(\vz, \vx)\right\|_{X}^{(1)} \cdot \underbrace{\left\|\eta \int_{t_s}^T \ell'(f_{\tau}(\vx)) dt\right\|_{X}^{(\infty)}}_{\text{independent of $\vz$}}.\label{eq: gpdot bound rhs}
\end{align}
The lazy learning assumption leads that 
\begin{align*}
 \left\|\Theta_\tau(\vz, \vx)\right\|_{X}^{(1)} 
&\leq  \E_\vx\left[\tr\left(\Theta_\tau(\vz, \vx)^\top \Theta_\tau(\vz, \vx)\right)^{1 / 2}\right] \\
&\leq  \E_\vx\left[\tr\left(|\Theta_\tau(\vz, \vx)|\right)\right] \\
&\leq  \E_\vx\left[\tr\left(|\Theta_T(\vz, \vx)|\right)\right]+o\delta \\
&=  \tr\left(\E_\vx\left[\left|\Theta_T(\vz, \vx) \right|\right]\right)+o\delta. 
\end{align*} 
Again the lazy learning assumption for $|\Theta_T(\vz,\vz)|=\Theta_T(\vz,\vz)$ ensures that 
\begin{align*}
 \left\|\nabla_\vtheta f_\tau(\vz) (\vtheta_T - \vtheta_{t_s})\right\|^2
& =
 (\vtheta_T - \vtheta_{t_s})^T\Theta_\tau(\vz,\vz)(\vtheta_T - \vtheta_{t_s})\\
&\geq 
 (\vtheta_T - \vtheta_{t_s})^T (\Theta_T(\vz,\vz)-\delta\mbox{\boldmath $I$}) (\vtheta_T - \vtheta_{t_s})\\
&=
 \left\|\nabla_\vtheta f_T(\vz) (\vtheta_T - \vtheta_{t_s})\right\|^2  -\delta\|\vtheta_T - \vtheta_{t_s}\|^2. 
\end{align*}
Hence, we have
\begin{align*}
 \|\nabla_\vtheta f_T(\vz) (\vtheta_T - \vtheta_{t_s})\|
 &\leq  
 \sqrt{\|\nabla_\vtheta f_\tau(\vz) (\vtheta_T - \vtheta_{t_s})\|^2+\delta\|\vtheta_T - \vtheta_{t_s}\|^2} \\
 &\leq
 \|\nabla_\vtheta f_\tau(\vz) (\vtheta_T - \vtheta_{t_s})\| + \sqrt{\delta}\|\vtheta_T - \vtheta_{t_s}\|. 
\end{align*}
Substituting the above inequalities into \eqref{eq: gpdot bound rhs}, we obtain the conclusion of the lemma. 
\end{proof}

By setting $\delta=0$ in the above lemma, we are able to obtain~\eqref{eq: bound-practical} with A5 and~\eqref{eq: 3-NTKs}.

\subsection{Proof of Proposition~\ref{prop: Ozz}}
\label{appendix: proof-4.1}

We proof the proposition with an optional, diagonal scaling matrix $\mathbf{\Gamma}$, i.e., by modifying~\eqref{eq: Layer-wise scaling} as:
\begin{equation}
    \label{eq: true Layer-wise scaling}
    \nabla_\vtheta \ef_T(\vz) \mathbf{\Gamma}^2 \nabla_\vtheta \ef_T(\vx)^\top \approx \Theta(\vz, \vx).
\end{equation}

\begin{proposition}
 (Proposition~\ref{prop: Ozz} with scaling)
 Suppose that $\ef$ is $\gamma$-smooth w.r.t. $\vtheta$, i.e.,
 $$\|\nabla_\vtheta\ef(\vz;\vtheta)-\nabla_\vtheta\ef(\vz;\vtheta')\|_{\mathrm{F}}\leq \gamma\|\vtheta-\vtheta'\|. $$
 Let $\rvv$ be a random variable such that
 $\E_{\rvv}[\rvv] = \boldsymbol{0}, \E_{\rvv}[\rvv \rvv^\top] = \epsilon^2 \mI$ and 
 $\E_{\rvv}[\|\rvv\|^k ] \leq  C_k\epsilon^k$ for $k=3,4$, where $C_k$ is a constant depending on $k$ and the dimension of $\rvv$. 
 Then, under A1, it holds that 
    \begin{equation}
        \label{eq: Ozz-trace-estimator replica}
	 \lim_{\epsilon \rightarrow 0} \frac{1}{\epsilon^2} 
        \E_{\rvv}\left[ \|
		 \ef(\vz; \vtheta_T + \mathbf{\Gamma} \rvv) - \ef(\vz; \vtheta_T)
		 \|^2 \right]
	=
	\tr \left(
        \nabla_\vtheta \ef(\vz; \vtheta_T)
        \mGamma^2 \nabla_\vtheta \ef(\vz; \vtheta_T)^\top
	    \right). 
    \end{equation}
\end{proposition}
\begin{proof}
For each component $\ef_i, i=1,\ldots,o$, the mean value theorem leads that there exists $t_i\in[0,1]$ such that 
\begin{align*}
& \phantom{=}
 |\ef_i(\vz; \vtheta_T + \mathbf{\Gamma} \rvv) - \ef_i(\vz; \vtheta_T)-\nabla_{\vtheta}\ef_i(\vz; \vtheta_T)^\top \Gamma\rvv|\\
&  = 
 |\nabla_{\vtheta}\ef_i(\vz; \vtheta_T + t_i \Gamma\rvv)^\top \Gamma\rvv-\nabla_{\vtheta}\ef_i(\vz; \vtheta_T)^\top \Gamma\rvv|\\
&  \leq
 \gamma\|\mathbf{\Gamma}\|^2\|\rvv\|^2. 
\end{align*}
For real numbers $a_i, b_i, i=1,\ldots,o$, suppose $|a_i-b_i|\leq c$. Then, we have
\begin{align*}
 \bigg|\sum_{i}a_i^2-\sum_{i}b_i^2\bigg|=\bigg|2\sum_i b_i(a_i-b_i)+\sum_i(a_i-b_i)^2\bigg|\leq 
 2c\sum_i|b_i|+oc^2. 
\end{align*}
Using the above inequality, we obtain
\begin{align}
&\phantom{=}
 \bigg|\E_{\rvv}\left[ \|        \ef(\vz; \vtheta_T + \mathbf{\Gamma} \rvv) - \ef(\vz; \vtheta_T)    \|^2 \right] 
-
 \E_{\rvv} \left[
 \rvv^\top \mathbf{\Gamma} \nabla_\vtheta \ef(\vz; \vtheta_T)^\top
      \nabla_\vtheta \ef(\vz; \vtheta_T) \mathbf{\Gamma} \rvv 
   \right]\bigg|\nonumber \\
&\leq 
 2\gamma\|\mathbf{\Gamma}\|^2
 \E_{\rvv}[\sum_{i}|\rvv^\top \mathbf{\Gamma} \nabla_\vtheta \ef_i(\vz; \vtheta_T)|\|\rvv\|^2]
 +o\gamma^2\|\mathbf{\Gamma}\|^4 \E_{\rvv}[\|\rvv\|^4] \nonumber\\
&\leq 
 2\sqrt{o}\gamma\|\mathbf{\Gamma}\|^3\|\nabla_\vtheta \ef(\vz; \vtheta_T)\|_{\mathrm{F}}
 \E_{\rvv}[\|\rvv\|^3]
 +o\gamma^2\|\mathbf{\Gamma}\|^4 \E_{\rvv}[\|\rvv\|^4]. 
    \label{eq: 4.1-1}
\end{align}
Note the cyclic trick for the trace ensures that 
\begin{align}
  &\phantom{=} 
   \E_{\rvv} \left[
 \rvv^\top \mathbf{\Gamma} \nabla_\vtheta \ef(\vz; \vtheta_T)^\top
      \nabla_\vtheta \ef(\vz; \vtheta_T) \mathbf{\Gamma} \rvv 
   \right]\nonumber \\
   & =  \E_{\rvv} \left[ \tr \left(
        \mGamma \nabla_\vtheta \ef(\vz; \vtheta_T)^\top
        \nabla_\vtheta \ef(\vz; \vtheta_T) \mGamma \rvv
        \rvv^\top 
    \right) \right] \nonumber \\
 &=  \tr \left( 
        \mGamma \nabla_\vtheta \ef(\vz; \vtheta_T)^\top
        \nabla_\vtheta \ef(\vz; \vtheta_T) \mGamma
        \E_{\rvv} \left[
            \rvv \rvv^\top
    \right] \right) \nonumber \\
 &=  \tr \left(
        \mGamma \nabla_\vtheta \ef(\vz; \vtheta_T)^\top
        \nabla_\vtheta \ef(\vz; \vtheta_T) \mGamma
        \cdot \epsilon^2 \mI
    \right) \label{eq: 4.1-2} \\
 &=  \epsilon^2 \tr \left(
        \nabla_\vtheta \ef(\vz; \vtheta_T)
        \mGamma^2 \nabla_\vtheta \ef(\vz; \vtheta_T)^\top
    \right) \label{eq: 4.1-3}.
\end{align}
In~\eqref{eq: 4.1-2} we applied the condition that $\E_{\rvv} [ \rvv \rvv^\top] = \epsilon^2 \mI$.
We note that this is a slightly modified version of the well-known Hutchinson's Trace Estimator. We refer the readers to the existing analysis of such estimators~\cite{avron2011randomized} for more details.
As a result, we obtain
{\small\begin{align*}
&\phantom{=}
 \left|
\lim_{\epsilon\rightarrow0} \frac{1}{\epsilon^2}
\E_{\rvv}\left[ \|        \ef(\vz; \vtheta_T + \mathbf{\Gamma} \rvv) - \ef(\vz; \vtheta_T)    \|^2 \right] 
 -
 \tr \left(
        \nabla_\vtheta \ef(\vz; \vtheta_T)
        \mGamma^2 \nabla_\vtheta \ef(\vz; \vtheta_T)^\top
    \right)
 \right|\nonumber \\
&\leq 
 \lim_{\epsilon\rightarrow0}
 2\sqrt{o}\gamma\|\mathbf{\Gamma}\|^3\|\nabla_\vtheta \ef(\vz; \vtheta_T)\|_{\mathrm{F}} 
 C_3\epsilon
 +o\gamma^2\|\mathbf{\Gamma}\|^4 C_4\epsilon^2\\
&=0. 
\end{align*}}%
The above equality means the conclusion of the proposition. 
\end{proof}

\subsection{Additional derivations for Section~\ref{subsec: surrogate-ensemble}}
\label{appendix: detail-4.3}

Under the distribution of $\rvv$ we have
\begin{align*}
\E_\rvv[f^{\mathrm{emp}}(\vz;{\bm\theta}_T+\Gamma\rvv)] 
& =
 f^{\mathrm{emp}}(\vz;{\bm\theta}_T)
 +
 \underbrace{\E_\rvv[\nabla_{\theta}f^{\mathrm{emp}}(\vz;{\bm\theta}_T)^T\Gamma\rvv]}_{=0\   \text{from}\   \E_\rvv[\rvv]=0}
 +
 O(\E[\rvv^2])\\
& =
 f^{\mathrm{emp}}(\vz;{\bm\theta}_T) + O(\epsilon^2),
\end{align*}
which indicates that $\E_\rvv[f^{\mathrm{emp}}(\vz;{\bm\theta}_T+\Gamma\rvv)] \approx f^{\mathrm{emp}}(\vz;{\bm\theta}_T)$ when $\epsilon$ is small.

We continue by the computation of $\mathrm{Tr}\Var[\widetilde{f}^{\mathrm{raw}}(\vz)]$: 
\begin{align*}
\mathrm{Tr}\Var_\rvv[\widetilde{f}^{\mathrm{raw}}(\vz)] 
& =
 \E_\rvv[\|f^{\mathrm{emp}}(\vz;{\bm\theta}_T+\Gamma\rvv)-\E_\rvv[f^{\mathrm{emp}}(\vz;{\bm\theta}_T+\Gamma\rvv)]\|^2]  \\
& =
 \E_\rvv[\|f^{\mathrm{emp}}(\vz;{\bm\theta}_T+\Gamma\rvv)-f^{\mathrm{emp}}(\vz;{\bm\theta}_T)+O(\epsilon^2)\|^2] \\
& =
 \E_\rvv[\|f^{\mathrm{emp}}(\vz;{\bm\theta}_T+\Gamma\rvv)-f^{\mathrm{emp}}(\vz;{\bm\theta}_T)\|^2]+O(\epsilon^4) \\
& \approx 
 \epsilon^2\mathrm{Tr}\Theta(\vz,\vz) + O(\epsilon^4).
\end{align*}

Let $\widetilde{\Theta}_{\mathrm{Tr}}(\vz, \vz)$ be an approximation of $\epsilon^2 \Tr \Theta(\vz, \vz)$, which is being computed empirically in line 11 of Alg.~\ref{alg: TULiP}.

Thus, $\gamma^2\mathrm{Tr}\Var_\rvv[\widetilde{f}^{\mathrm{raw}}(\vz)]$ reads: 
\begin{align*}
 \gamma^2\mathrm{Tr}\Var_\rvv[\widetilde{f}^{\mathrm{raw}}(\vz)]
& =
 \frac{[\widetilde{\Theta}_{\mathrm{Tr}}(\vz,\vz)-\lambda D]_+}{\widetilde{\Theta}_{\mathrm{Tr}}(\vz,\vz)}\mathrm{Tr}\Var_\rvv[\widetilde{f}^{\mathrm{raw}}(\vz)]\\
& \approx 
 \frac{[\widetilde{\Theta}_{\mathrm{Tr}}(\vz,\vz)-\lambda D]_+}{\widetilde{\Theta}_{\mathrm{Tr}}(\vz,\vz)}
 (\widetilde{\Theta}_{\mathrm{Tr}}(\vz,\vz) + O(\epsilon^4))\\
& \approx
 [\epsilon^2 \Tr \Theta(\vz,\vz)-\lambda D]_+ + O(\epsilon^4)
\end{align*}
where $[\;\cdot\;]_+$ denotes $\mathrm{max}(\cdot, 0)$.

For $D$, from approximation~\eqref{eq: D-approximation} we have
\begin{align*}
    D =& \left\| \vphantom{\widetilde{\Theta}} \ef(\vz; \vtheta_T + \epsilon \delta \mGamma (\vtheta_T - \vtheta_{t_s}) ) - \ef(\vz; \vtheta_T) \right\| \\
    \approx &
    \epsilon \delta \left\|
        \nabla_\vtheta f_T(\vz) (\vtheta_T - \vtheta_{t_s})
    \right\|.
\end{align*}

As a result, we have 
\begin{align*}
 \gamma^2\mathrm{Tr}\Var[\widetilde{f}^{\mathrm{raw}}(\vz)]
 \approx
 \left[
 \epsilon^2 \Tr \Theta(\vz,\vz) - \epsilon \delta \left\|
        \nabla_\vtheta f_T(\vz) (\vtheta_T - \vtheta_{t_s})
    \right\|
 \right]_+. 
\end{align*}

Recall that~\eqref{eq: variance-target} indicates that
\begin{equation*}
    \tr(\Var_{\Delta f}[
        \hat{f}_T(\vz)
    ]) \leq
    \E_{\Delta f}[
        \| \hat{f}_T(\vz) - f_T(\vz) \|^2
    ],
\end{equation*}
and Prop.~\ref{prop: Ozz} shows that
\begin{equation*}
    \| f_T(\vz)\! - \! \hat{f}_T(\vz) \| \lesssim  \big[\tr(\Theta(\vz, \vz)+ \underbrace{\E_x[\Theta(\vx, \vx)]}_{\text{Independent of $\vz$}}) - 2 K\left\|\nabla_\vtheta f_T(\vz)\left(\vtheta_T\!-\!\vtheta_{t_s}\right)\right\| \big]^{1/2}.
\end{equation*}

\subsection{Perturb-then-train and \eqref{eq: GT-uncertainty}}
\label{appendix: eq-1-eq-5-relation}

\cite{jacot2018neural} consider neural networks in an infinite-width limit with specified initialization scheme, which we have referred as the lazy limit in Section~\ref{sec: method}.
Under such limit, the linearized network~\eqref{eq: linearized-net} is justified as the empirical NTK (at initialization) converges to a specific deterministic kernel $\Theta$, where the distribution of a neural network $f(x; \vtheta)$'s initialization functional $f_\mathrm{Init}(\vx)$ converges to a Gaussian Process (NNGP)~\cite{lee2018deep}.
In~\eqref{eq: linearized-net}, it is equivalent to a deterministic (fixed) $\left.\nabla_\theta f_\mathrm{True}(x)\right|_{\theta = \theta^\ast}$ and a stochastic $f_\mathrm{Init}$ following the NNGP.

Using the model defined in~\eqref{eq: linearized-net} and the training process described in~\eqref{eq: Linearized dynamics}, \eqref{eq: GT-uncertainty} effectively becomes:
\begin{equation}
    \label{eq: init-functional-GT-uncertainty}
    \mathrm{Var}_{f_\mathrm{Init} \sim \mu_\mathrm{NNGP}}[f_T(x;\theta | \mathrm{Init} = f_\mathrm{Init})],
\end{equation}
where $f_T(x;\theta | \mathrm{Init} = f_\mathrm{Init})$ indicates a network trained via~\eqref{eq: Linearized dynamics} by time $T$, with $f_\mathrm{Init}$ as initialization.

When we set $t_s = 0$ (the initialization time), the perturbation $\Delta f$ will be applied to $f_\mathrm{Init}$. Therefore, given a fixed initialization $f_0$ to perturb, Theorem~\ref{theorem: direct gronwall} gives an upper-bound over a perturbation of the initialization functional:
\begin{equation}
    \label{eq: perturbed-init-GT-uncertainty}
    \mathrm{Var}_{\Delta f}[f_T(x;\theta | \mathrm{Init} = f_0 + \Delta f)],
\end{equation}
since $\hat{f}_T$ is supposed to be trained from initialization $f_0 + \Delta f$, we have $\hat{f}_T = f_T(x;\theta | \mathrm{Init} = f_0 + \Delta f)$, hence the above.

Comparing it to~\eqref{eq: init-functional-GT-uncertainty}, we see that the difference between them is the distribution of the initialization functional $f_\mathrm{Init}$. In~\eqref{eq: init-functional-GT-uncertainty}, $f_\mathrm{Init}$ distributes according to the NNGP; while in~\eqref{eq: perturbed-init-GT-uncertainty}, it is centered around $f_0$ with a stochastic perturbation $\Delta f$. Intuitively, by using theorem~\ref{theorem: direct gronwall}, we approximate the predictive variance trained from the NNGP prior with the predictive variance trained from a random perturbed initialization $f_0 + \Delta f$. Figure~\ref{fig: toy-datasets} visualizes such an approximation.

\section{Details of Experimental Setup}

\subsection{Dataset Descriptions}
\label{appendix: datasets}

An overview of all considered datasets is provided below.
ID and OOD dataset setups are summarized in Table~\ref{table: id-ood-val-pairs}.
Please refer to~\cite{zhang2023openoodv15} for more details.

\subsubsection{ID Datasets}
\subsubsection{CIFAR-10}
The CIFAR-10 training set~\cite{krizhevsky2009learning} consists 60000 $32 \times 32$ colored images, containing 10 classes of \textit{airplane}, \textit{automobile}, \textit{bird}, \textit{cat}, \textit{deer}, \textit{dog}, \textit{frog}, \textit{horse}, \textit{ship} and \textit{truck}.
The test set originally contained 10000 images from the same classes, where we separated 1000 validation images and 9000 test images from 
the original test set following \cite{zhang2023openoodv15}.
The dataset and each split are even in classes.

\subsubsection{CIFAR-100}
CIFAR-100~\cite{krizhevsky2009learning} contains 60000 $32 \times 32$ images sampled from 100 classes, covering a wider range of images beyond CIFAR-10. 
Similar to CIFAR-10, 1000 images are taken out from the ID test set, forming a validation set.

\subsubsection{ImageNet-1K}
ImageNet-1K~\cite{5206848}, also known as ILSVRC 2012, spans 1000 object classes and contains 1,281,167 training images, 50,000 validation images and 100,000 test images, each of size $224 \times 224$. In the OpenOOD setup, 45,000 validation images are used as ID test and 5,000 as ID validation.

\subsubsection{ImageNet-200}
ImageNet-200 \cite{zhang2023openoodv15} is a 200-class subset of ImageNet-1K compiled in OpenOOD version 1.5, with 10,000 $224 \times 224$ validation images.

\subsubsection{Semantic-Shift OOD Datasets}
\subsubsection{Tiny-ImageNet}
Tiny-ImageNet~\cite{le2015tiny} has 100,000 images divided up into 200 classes, each with 500 training images, 50 validating images, and 50 test images. Compared to ImageNet-200, every image in Tiny-ImageNet is downsized to a 64×64 coloured image.
\subsubsection{MNIST}
Modified National Institute of Standards and Technology database~\cite{726791} contains 60,000 training and 10,000 test images of handwritten digits. Each image is anti-aliased, normalized and centered to fit into a 28x28 pixel bounding box.
\subsubsection{SVHN}
Street View House Number~\cite{37648} dataset contains house numbers that are captured on Google Street View, consisting of 73257 digits for training, and 26032 digits for testing. In our setup, we used the MNIST-like 32-by-32 format, centered around a single character.
\subsubsection{Textures}
Describable Textures Dataset~\cite{6909856} is a set of 47 categories of textures, collected from Google and Flickr via relevant search queries. It has 5640 images, 120 images for each category, where the sizes range between 300x300 and 640x640.
\subsubsection{Places365}
Places365~\cite{7968387} is a scene recognition dataset. The standard version is composed of 1.8 million train and 36000 validation images from 365 scene classes.
\subsubsection{NINCO}
No ImageNet Class Objects~\cite{bitterwolf2023outfixingimagenetoutofdistribution} consists of 5879 samples from 64 OOD classes. These OOD classes were selected to have no categorical overlap with any classes of ImageNet-1K. Each sample was inspected individually by the authors to not contain ID objects.
\subsubsection{SSB-Hard}
Semantic Shift Benchmark-Hard~\cite{vaze2022openset} split contains 49,000 images across 980 categories of ImageNet-21K that has a short total semantic distance.
\subsubsection{iNaturalist}
The iNaturalist dataset~\cite{8579012} has 579,184 training and 95,986 validation images from 5,089 different species of plants and animals.
\subsubsection{OpenImage-O}
OpenImage-O~\cite{kuznetsova2020open} is image-by-image filtered from the test set of OpenImage-V3, which has been collected from Flickr without a predefined list of class names or tags. In the OpenOOD setup, 1,763 images are picked out as validation OOD.

\subsubsection{Covariate-Shift OOD Datasets}
\subsubsection{Blur-ImageNet} This blurred ImageNet dataset contains ImageNet images with a Gaussian blur of $\sigma = 2$. The same splits are used as in the above description in the ImageNet-1K section.

\subsubsection{ImageNet-C}
ImageNet-C~\cite{hendrycks2018benchmarking} has 15 synthetic corruption types (such as noise, blur, pixelate) on the standard ImageNet-1K, each with 5 severities. In OpenOOD, 10,000 images are randomly sampled uniformly across the 75 combinations to form the test set.
\subsubsection{ImageNet-R}
ImageNet-R~\cite{Hendrycks_2021_ICCV} contains 30,000 images of different renditions of 200 ImageNet classes, such as art, graphics, patterns, toys, and video games.

\subsubsection{ImageNet-ES}
ImageNet-ES~\cite{baek2024unexplored} consists of 202,000 photos of images from Tiny-ImageNet. Each image is displayed on screen with high fidelity and photographed in a controlled environment with different parameter settings. We only used the 64,000 photos in the test set.

\subsection{Hyper-parameters}
\label{appendix: hparams}

In practice, when handling hyperparameters, we found it beneficial to first search for an optimal value for $J_\mathrm{scaling}$, the most important parameter of TULiP as pointed out in Sec .~\ref {sec: experiments}.
It controls the overall strength of TULiP and may depends on the network architecture and training scheme.
Another important parameter is $\epsilon$, though it has relatively less impact as long as the perturbation is not overwhelming the network's original weights.
If one's computational resource allows for further exploration, optimal values of $\lambda$ and $\delta$ can be searched for better performance.
Typically, when increasing $\delta$, $\lambda$ should be decreased as they have a multiplicative relationship.
If a validation set is not available, one may either use the suggested value or investigate network outputs after weight perturbation.
When the network output becomes senseless after perturbation (e.g., a prediction close to random-guessing), it often indicates that $\epsilon$ or $J_\mathrm{scaling}$ is too large.

\subsubsection{Grid Search}
Table~\ref{table: hparam-range} lists the hyper-parameter search range for all considered methods on the validation set.

\subsubsection{$\sqrt{o}$ scaling of $\lambda$}
In practice, when the number of network output dimensions $o$ varies, we found that a $\sqrt{o}$ scaling of $\lambda$-term ($D$) works more consistently.
This might be due to our choice of using the computationally friendly Frobenius norm in Theorem~\ref{theorem: direct gronwall} instead of a tighter spectrum norm.
It is further explained in the proof listed in Sec.~\ref{appendix: proof-3.1}.

\subsubsection{Hardware}
Each of our experiments is conducted on a single-node machine using an NVIDIA A6000 GPU.

\begin{table}[t]
    \centering
    \caption{ID, OOD and OOD-val dataset setups.}
    \label{table: id-ood-val-pairs}
    \begin{adjustbox}{width=\linewidth}
    \begin{tabular}{l c c c c} 
    \toprule
        ID Dataset & near-OOD & far-OOD & near/far-OOD Validation & Cov-Shift OOD \\
        \midrule
        \multirow{4}{*}{CIFAR-10} &  & MNIST & \multirow{4}{*}{Tiny-ImageNet} & \\
        & CIFAR-100 & SVHN & & \\
        & Tiny-ImageNet & Textures & & \\
        & & Places365 & & \\
        \midrule
        \multirow{4}{*}{CIFAR-100} & & MNIST & \multirow{4}{*}{Tiny-ImageNet} & \\
        & CIFAR-10 & SVHN & & \\
        & Tiny-ImageNet & Textures & & \\
        & & Places365 & & \\
        \midrule
        &  & iNaturalist & \multirow{4}{*}{OpenImage-O} & Blur-ImageNet \\
        ImageNet-1K & SSB-Hard & Textures & & ImageNet-C \\
        ImageNet-200 & NINCO & OpenImage-O & & ImageNet-R \\
        & & & & ImageNet-ES \\
    \bottomrule
    \end{tabular}\end{adjustbox}
\end{table}

\begin{table}[t]
    \centering
    \caption{Hyper-parameter (available at evaluation time) search ranges.}
    \label{table: hparam-range}
    \begin{tabular}{l c} 
    \toprule
      Method                     & Hyper-parameters \\
      \midrule
      MC-Dropout                 & N/A \\
      MDS                        & N/A \\
      MLS                        & N/A \\
      EBO                        & Temperature: $\{1\}$ \\
      \midrule
      ViM                        & Dimension: $\{256, 1000\}$ \\
      ASH                        & Percentile: $\{65, 70, 75, 80, 85, 90, 95\}$ \\ 
      \midrule
      ODIN                       & \makecell{Temperature: $\{1, 10, 100, 1000\}$ \\ Noise: $\{0.0014, 0.0028\}$} \\
      \midrule
      TULiP                      & $J_\mathrm{scaling}: \{1.0, 1.25, 1.5, 1.75, 2.0\}$
      \\ \bottomrule
    \end{tabular}
\end{table}

\subsection{Time Complexity of TULiP}
As listed in Algorithm~\ref{alg: TULiP}, TULiP requires $\mathcal{O}(M)$ forward passes to evaluate a minibatch of test data. Compared to single-pass methods, such limitation renders TULiP ineffective despite its performance as shown in Sec.~\ref{sec: experiments}, since forwarding a network could be expensive as networks grow in size. Nevertheless, TULiP is not $\mathcal{O}(M)$ times slower than single-pass methods as forward evaluation is not the sole bottleneck of inferencing. Table~\ref{table: wall-clock-time} compares the wall-clock inference time of TULiP and EBO (a single-pass method) in our OOD setting.

\begin{table}[t]
    \centering
    \caption{Wall-clock time of our OOD experiments, contains a serial sequence of inference on ID (top row) and all corresponding OOD datasets (near and far).}
    \label{table: wall-clock-time}
    \begin{tabular}{r c c c c c} 
    \toprule
      Method & Forward passes & CIFAR-10 & ImageNet-200 & ImageNet-1K \\
    \midrule
      EBO & 1 & 44.32s & 112.37s & 3m 12.60s \\
      \multirow{2}{*}{TULiP} & \multirow{2}{*}{$\mathcal{O}(M)$, $M = 10$} & 96.30s & 190.41s & 10m 59.24s \\
      & & {\small(2.17x)} & {\small(1.69x)} & {\small(3.42x)} \\
    \bottomrule
    \end{tabular}
\end{table}

\end{document}